\RequirePackage{fix-cm}
\documentclass[twocolumn]{svjour3}          
\smartqed  
\usepackage{graphicx}
\usepackage{amsmath,amssymb,amsfonts}
\usepackage{algorithm,algorithmic}
\usepackage[usenames]{color}
\usepackage[final,notref,notcite]{showkeys}
\usepackage{multirow}
\usepackage{subfigure}

\newcommand{\prox}{\mathrm{prox}}
\newcommand{\supp}{\mathrm{supp}}

\newcommand{\bm}[1]{\boldsymbol{#1}}

\newcommand{\mbr}{\mathbb{R}}

\newcommand{\mbfi}{\mathbf{I}}
\newcommand{\mbfp}{\mathbf{P}}
\newcommand{\mbfu}{\mathbf{U}}
\newcommand{\mbfv}{\mathbf{V}}
\newcommand{\mbfw}{\mathbf{W}}
\newcommand{\mbfx}{\mathbf{X}}

\newcommand{\mce}{\mathcal{E}}

\newcommand{\mci}{\mathcal{I}}

\newcommand{\mcn}{\mathcal{N}}

\newcommand{\mcs}{\mathcal{S}}
\newcommand{\mcu}{\mathcal{U}}

\newcount\refnum\refnum=0
\def\myref{{\global\advance\refnum by 1} {\bf \large Lecture \the \refnum. }}
\newcommand{\bfx}{{\bf x}}
\newcommand{\bfb}{{\bf b}}

\newcommand{\bfe}{{\bf e}}
\newcommand{\bfg}{{\bf g}}

\newcommand{\bfz}{{\bf z}}
\newcommand{\bfu}{{\bf u}}
\newcommand{\bfv}{{\bf v}}
\newcommand{\bfw}{{\bf w}}

\newcommand{\sign}{{\mathrm{sign}}}

\journalname{Pattern Anal Applic}
\begin{document}

\title{Proximal gradient method for huberized support vector machine
}

\titlerunning{Proximal gradient for huberized SVMs}        

\author{Yangyang Xu         \and
        Ioannis Akrotirianakis	  \and
	Amit Chakraborty
}


\institute{Yangyang Xu \at
              Department of Computational and Applied Mathematics, Rice University, 		              Houston, TX. \\
              \email{yangyang.xu@rice.edu}            \\
           \and
           Ioannis Akrotirianakis \at
              Siemens Corporate Research, Princeton, NJ. \\
		\email{ioannis.akrotirianakis@siemens.com}  \\
	   \and
	  Amit Chakraborty \at
	 Siemens Corporate Research, Princeton, NJ. \\
		\email{amit.chakraborty@siemens.com}  \\
}

\date{Received: date / Accepted: date}

\maketitle

\begin{abstract}
The Support Vector Machine (SVM) has been used in a wide variety of classification problems. The original SVM uses the hinge loss function, which is non-differentiable and makes the problem difficult to solve in particular for regularized SVMs, such as with $\ell_1$-regularization. This paper considers the Huberized SVM (HSVM), which uses a differentiable approximation of the hinge loss function. We first explore the use of the Proximal Gradient (PG) method to solving binary-class HSVM (B-HSVM) and then generalize it to multi-class HSVM (M-HSVM). Under strong convexity assumptions, we show that our algorithm converges linearly. In addition, we give a finite convergence result about the support of the solution, based on which we further accelerate the algorithm by a two-stage method. We present extensive numerical experiments on both synthetic and real datasets which demonstrate the superiority of our methods over some state-of-the-art methods for both binary- and multi-class SVMs.
\keywords{Support Vector Machine \and Proximal Gradient method \and Binary and multiclass classification problems \and Huberized hinge loss \and Elastic net \and Linear convergence.}
\end{abstract}

\section{Introduction}
The original linear support vector machine (SVM) aims to find a hyperplane that separates a collection of data points. Since it was proposed in \cite{cortes1995support}, it has been widely used for binary classifications, such as texture classification \cite{kim2002support}, gene expression data analysis \cite{brown2000knowledge}, face recognition \cite{heisele2001face}, to name a few. Mathematically, given a set of samples $\{\bfx_i\}_{i=1}^n$ in $p$-dimensional space and each $\bfx_i$ attached with a label $y_i\in\{+1,-1\}$, the linear SVM learns a hyperplane $(\bfw^*)^\top\bfx+b^*=0$ from the training samples.
A new data point $\bfx$ can be categorized into the ``$+1$'' or ``$-1$'' class by inspecting the sign of $(\bfw^*)^\top\bfx+b^*$.

The binary-class SVM (B-SVM) has been generalized to multicategory classifications to tackle problems that have data points belonging to more than two classes. 
The initially proposed multi-class SVM (M-SVM) methods construct several binary classifiers, such as ``one-against-all' \cite{bottou1994comparison}, ``one-against-one'' \cite{knerr1990single}, and ``directed acyclic graph SVM'' \cite{platt2000large}. 
These M-SVMs may suffer from data imbalance, namely, some classes have much fewer data points than others, which can result in inaccurate predictions. One alternative is to put all the data points together in one model, 
which results in the so-called ``all-together'' M-SVMs. The ``all-together'' M-SVMs train multi-classifiers by considering a single large optimization problem. An extensive comparison of different M-SVMs can be found in \cite{hsu2002comparison}.  

In this paper, we consider both B-SVM and M-SVM. More precisely, we consider the binary-class huberized SVM (B-HSVM) (see \eqref{eq:main} below) for B-SVM and the ``all-together'' multi-class HSVM (M-HSVM) (see \eqref{eq:multihsvm} below) for M-SVM. The advantage of HSVM over classic SVM with hinge loss is the continuous differentiability of its loss function, which enables the use of the ``fastest'' first-order method: Proximal Gradient (PG) method \cite{nesterov2007gradient,BeckTeboulle2009} (see the overview in section \ref{sec:pg}). We demonstrate that the PG method is in general much faster than existing methods for (regularized) B-SVM and M-SVM while yielding comparable prediction accuracies. In addition, extensive numerical experiments are done to compare our method to state-of-the-art ones for both B-SVM and M-SVM on synthetic and also benchmark datasets. Statistical comparison is also performed to show the difference between the proposed method and other compared ones.

\subsection{Related work}
B-HSVM appears to be first considered\footnote{Strictly speaking, we add the term $\frac{\lambda_3}{2}b^2$ in \eqref{eq:main}. This modification is similar to that used in \cite{keerthi2006modified}, and the extra term usually does not affect the prediction but makes PG method converge faster.} by Wang \emph{et al} in \cite{wang2008hybrid}. They 
demonstrate that B-HSVM can perform better than the original unregularized SVM and also $\ell_1$-regularized SVM (i.e., \eqref{eq:elastic} with $\lambda_2=0$) for microarray classification. 

A closely related model to B-HSVM is the elastic net regularized SVM \cite{wang2006doubly}
\begin{equation}\label{eq:elastic}
\min_{b,\bfw}\frac{1}{n}\sum_{i=1}^n\left[1-y_i(b+\bfx_i^\top \bfw)\right]_++\lambda_1\|\bfw\|_1+\frac{\lambda_2}{2}\|\bfw\|_2^2,
\end{equation}
where $[1-t]_+=\max(0,1-t)$ is the hinge loss function. The $\ell_1$-norm regularizer has the ability to perform continuous shrinkage and automatic variable selection at the same time \cite{tibshirani1996shrinkageLasso}, and the $\ell_2$-norm regularizer can help to reduce the variance of the estimated coefficients and usually results in satisfactory prediction accuracy, especially in the case where there are many correlated variables. The elastic net inherits the benefits of both $\ell_1$ and $\ell_2$-norm regularizers and can perform better than either of them alone, as demonstrated in \cite{zou-hastie2005elastic-net}.

Note that \eqref{eq:elastic} uses non-differentiable hinge loss function while B-HSVM uses differentiable huberized hinge loss function. The differentiability makes B-HSVM relatively easier to solve. A path-following algorithm was proposed by Wang \emph{et. al} \cite{wang2008hybrid} for solving B-HSVM. Their algorithm is not efficient in particular for large-scale problems, since it needs to track the disappearance of variables along a regularization path. Recently, Yang and Zou \cite{yangefficient} proposed a Generalized Coordinate Descent (GCD) method, which was, in most cases, about 30 times faster than the path-following algorithm. However, the GCD method needs to compute the gradient of the loss function of B-HSVM after each coordinate update, which makes the algorithm slow. 

M-HSVM has been considered in \cite{li2010huberized}, which  generalizes the work \cite{wang2008hybrid} on B-HSVM to M-HSVM and also makes a path-following algorithm. 
However, their algorithm could be even worse since it also needs to track the disappearance of variables along a regularization path and M-HSVM often involves more variables than those of B-HSVM. Hence, it is not suitable for large-scale problems either.

Similar to M-HSVM, several other models have been proposed to train multiple classifiers by solving one single large optimization problem to handle the multi-category classification, such as the $\ell_1$-norm regularized M-SVM in \cite{wang20071}
and the $\ell_\infty$-norm regularized M-SVM in \cite{zhang2008variable}.
Again, these models use the non-differentiable hinge loss function and are relatively more difficult than M-HSVM to solve. There are also methods that use binary-classifiers to handle multicategory classification problems including ``one-against-all'' \cite{bottou1994comparison}, ``one-against-one'' \cite{knerr1990single}, and ``directed acyclic graph SVM'' \cite{platt2000large}. The work \cite{galar2011overview} makes a thorough review of methods using binary-classifiers for multi-category classification problems and gives extensive experiments on various applications. In a follow up and more recent paper \cite{Gal.et.al.15} the same authors present a dynamic classifier weighting method which deals with the limitations introduced by the non-competent classifiers in the one-versus-one classification strategy. In order to dynamically weigh the outputs of the individual classifiers, they use the nearest neighbor of every class from the instance that needs to be classified. Furthermore in \cite{Kra.et.al.14} the authors propose a new approach for building multi-class classifiers based on the notion of data-clustering in the feature space. For the derived clusters they construct one-class classifiers which can be combined to solve complex classification problems.

One key advantage of our algorithm is its scalability with the training sample size, and thus it is applicable for large-scale SVMs. While preparing this paper, we note that some other algorithms are also carefully designed for handling the large-scale SVMs, including the block coordinate Frank-Wolfe method \cite{lacoste2012block} and the stochastic alternating direction method of multiplier \cite{ouyang2013stochastic}. In addition, Graphics Processing Unit (GPU) computing has been utilized in \cite{li2013parallel} to run multiple training tasks in parallel to accelerate the cross validation procedure. Furthermore, different variants of SVMs have been proposed for specific applications such as the Value-at-Risk SVM for stability to outliers \cite{tsyurmasto2014value}, the structural twin SVM to contain prior domain knowledge \cite{qi2013structural}, the hierarchical SVM for customer churn prediction \cite{chen2012hierarchical} and the ellipsoidal SVM for outlier detection in wireless sensor networks \cite{zhang2013distributed}. Finally in \cite{CzaTab14} a two-ellipsoid kernel decomposition is proposed for the efficient training of SVMs. In order to avoid the use of SOCP techniques, introduced by the ellipsoids, the authors transform the data using a matrix which is determined from the sum of the classes' covariances. With that transformation it is possible to use classical SVM, and their method can be incorporated into existing SVM libraries.

\subsection{Contributions}
We develop an efficient PG method to solve the B-HSVM 
\begin{equation}\label{eq:main}
\min_{b,\bfw}\frac{1}{n}\sum_{i=1}^n\phi_H\left(y_i(b+\bfx_i^\top \bfw)\right)+\lambda_1\|\bfw\|_1+\frac{\lambda_2}{2}\|\bfw\|_2^2+\frac{\lambda_3}{2}b^2,
\end{equation}
and the ``all-together'' M-HSVM
\begin{equation}\label{eq:multihsvm}
\begin{array}{cl}
\underset{\bfb,\mbfw}{\min}& \frac{1}{n}\overset{n}{\underset{i=1}{\sum}}\overset{J}{\underset{j=1}{\sum}} a_{ij}\phi_H(b_j+\bfx_i^\top\bfw_j)+\lambda_1\|\mbfw\|_1\\
&\hspace{1cm}+\frac{\lambda_2}{2}\|\mbfw\|_F^2+\frac{\lambda_3}{2}\|\bfb\|^2,\\[0.1cm]
\text{ s.t.}& \mbfw \bfe = \mathbf{0}, \bfe^\top\bfb = 0.
\end{array}
\end{equation}
In \eqref{eq:main}, $y_i\in\{+1,-1\}$ is the label of $\bfx_i$, and 
$$\phi_H(t)=\left\{
\begin{array}{ll}
0, & \text{ for }t>1,\\
\frac{(1-t)^2}{2\delta}, &\text{ for }1-\delta<t\le 1,\\
1-t-\frac{\delta}{2}, & \text{ for } t\le 1-\delta,
\end{array}
\right.$$ 
is the huberized hinge loss function which is continuously differentiable.
In \eqref{eq:multihsvm}, 
$y_i\in\{1,\ldots,J\}$ is the $i$-th label, $\|\mbfw\|_1=\sum_{i,j}|w_{ij}|$, $a_{ij}=1$ if $y_i\neq j$ and $a_{ij}=0$ otherwise, $\bfe$ denotes the vector with all one's, and $\bfw_j$ denotes the $j$-th column of $\mbfw$. The constraints $\mbfw \bfe = \mathbf{0}$ and $\bfe^\top\bfb = 0$ in \eqref{eq:multihsvm} are imposed to eliminate redundancy in $\mbfw,\bfb$ and also are necessary to make the loss function $$\ell_M=\frac{1}{n}\overset{n}{\underset{i=1}{\sum}}\overset{J}{\underset{j=1}{\sum}} a_{ij}\phi_H(b_j+\bfx_i^\top\bfw_j)$$ Fisher-consistent \cite{lee2004multicategory}.

We choose the PG methodology because it requires only first-order information and converges fast. As shown in \cite{nesterov2007gradient,BeckTeboulle2009}, it is an optimal first-order method. Note that the objectives in \eqref{eq:main} and \eqref{eq:multihsvm} have non-smooth terms and are not differentiable. Hence, direct gradient or second-order methods such as the interior point method are not applicable.  

We speed up the algorithm by using a two-stage method, which detects the support of the solution at the first stage and solves a lower-dimensional problem at the second stage. For large-scale problems with sparse features, the two-stage method can achieve more than 5-fold acceleration. We also analyze the convergence of PG method under fairly general settings and get similar results as those in \cite{nesterov2007gradient,schmidt2011convergence}.

In addition, we compare the proposed method to state-of-the-art ones for B-SVM and M-SVM on both synthetic and benchmark datasets. Extensive numerical experiments demonstrate that our method can outperform other compared ones in most cases. Statistical tests are also performed to show significant differences between the compared methods.

\subsection{Structure of the paper}
The rest of the paper is organized as follows. In section \ref{sec:algorithm}, we overview the PG method and then apply it to \eqref{eq:main} and \eqref{eq:multihsvm}. 
In addition, assuming strong convexity, we show linear convergence of the PG method under fairly general settings. 
Numerical results are given in section \ref{sec:numerical}. Finally, section \ref{sec:conclusion} concludes the paper.

\section{Algorithms and convergence analysis}\label{sec:algorithm}
In this section, we first give an  overview of the PG method. Then we apply it to \eqref{eq:main} and \eqref{eq:multihsvm}.  We complete this section by showing that under strong convexity assumptions the PG method possesses linear convergence.

\subsection{Overview of the PG method}\label{sec:pg}
Consider convex composite optimization problems in the form of
\begin{equation}\label{eq:comp}
\min_{\bfu\in\mcu}\xi(\bfu)\equiv \xi_1(\bfu)+\xi_2(\bfu),
\end{equation} 
where $\mcu\subset\mbr^m$ is a convex set, $\xi_1$ is a differentiable convex function with Lipschitz continuous gradient (that is, there exists $L>0$ such that $\| \nabla\xi_1(\bar{\bfu}) -\nabla\xi_1(\tilde{\bfu}) \| \le L\| \bar{\bfu} - \tilde{\bfu}\|$, for all $\bar{\bfu}, \tilde{\bfu} \in \mcu$), and $\xi_2$ is a proper closed convex function.
The PG method for solving \eqref{eq:comp} iteratively updates the solution by
\begin{equation}\label{eq:pg}
\bfu^k=\arg\underset{\bfu\in\mcu}{\min}~Q(\bfu,\hat{\bfu}^{k-1})
\end{equation}
where 
$$\begin{array}{ll}
Q(\bfu,\hat{\bfu}^{k-1})=&\xi_1(\hat{\bfu}^{k-1})+\langle\nabla \xi_1(\hat{\bfu}^{k-1}),\bfu-\hat{\bfu}^{k-1}\rangle\\[0.1cm]
&+\frac{L_k}{2}\|\bfu-\hat{\bfu}^{k-1}\|^2+\xi_2(\bfu),
\end{array}$$
$L_k>0$ and $\hat{\bfu}^{k-1}={\bfu}^{k-1}+\omega_{k-1}({\bfu}^{k-1}-{\bfu}^{k-2})$ for some nonnegative $\omega_{k-1}\le 1$. The extrapolation technique can significantly accelerate the algorithm. 


When $L_k$ is the Lipschitz constant of $\nabla \xi_1$, it can easily be shown that $\xi(\bfu)\le Q(\bfu,\hat{\bfu}^{k-1})$, and thus this method is a kind of majorization minimization, as illustrated in Figure \ref{fig:pg}. With appropriate choice of $\omega_{k-1}$ and $L_k$, the sequence $\{\xi(\bfu^k)\}$ converges to the optimal value $\xi^*=\xi(\bfu^*)$. Nesterov \cite{nesterov2007gradient}, and Beck and Teboulle \cite{BeckTeboulle2009} showed, independently, that if $\omega_{k-1}\equiv 0$ and $L_k$ is taken as the Lipschitz constant of $\nabla \xi_1$, then $\{\xi(\bfu^k)\}$ converges to $\xi^*$ with the rate $O(1/k)$, namely,
$\xi(\bfu^k)-\xi(\bfu^\ast)\le O(1/k).$ In addition, using carefully designed $\omega_{k-1}$, they were able to show that the convergence rate can be increased to $O(1/k^2)$, which is the optimal rate of first-order methods for general convex problems \cite{NesterovConvexBook2004}.

\begin{figure}\caption{Simple illustration of PG method: $Q(u,z)$ is a majorization approximation of $\xi$ at $z$, which is an extrapolated point of $x$ and $y$. The new iterate $u^*$ is the minimizer of $Q$.}\label{fig:pg}
\centering
\includegraphics[width=0.4\textwidth]{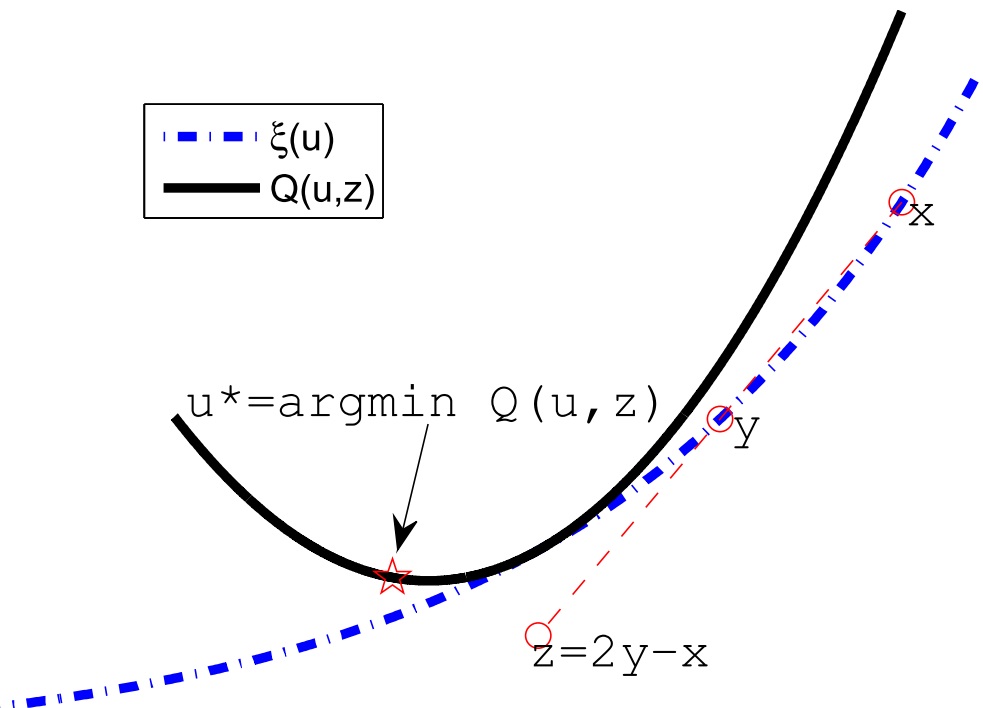}
\end{figure}


\subsection{PG method for binary-class HSVM}
We first write the B-HSVM problem \eqref{eq:main} into the general form of \eqref{eq:comp}. Let 
$$\left\{\begin{array}{l}
f_i(b,\bfw)=\phi_H\left(y_i(b+\bfx_i^\top \bfw)\right),\text{ for }i=1,\cdots,n,\\[0.2cm]
f(b,\bfw)=\frac{1}{n}\sum_{i=1}^nf_i(b,\bfw),\\[0.2cm]
g(b,\bfw)=\lambda_1\|\bfw\|_1+\frac{\lambda_2}{2}\|\bfw\|^2+\frac{\lambda_3}{2}b^2.
\end{array}\right.$$
Then \eqref{eq:main} can be written as
\begin{equation}\label{eq:eqmain}
\min_{b,\bfw}F(b,\bfw)\equiv f(b,\bfw)+g(b,\bfw).
\end{equation}
For convenience, we use the following notation in the rest of this section
\begin{equation}\label{notation_uv}\bfu=(b;\bfw),\qquad \bfv_i=(y_i;y_i\bfx_i).\end{equation}

\begin{proposition}\label{prop:f}
The function $f$ defined above is differentiable and convex, and its gradient $\nabla f$ is Lipschitz continuous with constant
\begin{equation}\label{lipconst}
L_f=\frac{1}{n\delta}\sum_{i=1}^n y_i^2(1+\|\bfx_i\|^2).
\end{equation}
\end{proposition}

The proof of this proposition involves standard arguments and can be found in the appendix.

Now, we are ready to apply PG to \eqref{eq:main}. Define the \emph{proximal operator} for a given extended real-valued convex function $h(\bfu)$ on $\mbr^m$ by
$$\prox_h(\bfv)=\arg\min_\bfu\frac{1}{2}\|\bfu-\bfv\|^2+h(\bfu).$$ 
Replacing $\xi_1$ and $\xi_2$ in \eqref{eq:pg} with $f$ and $g$, respectively, we obtain
the update
\begin{equation}\label{eq:mainupdate}
\left\{\begin{array}{l}b^k=\frac{L_k\hat{b}^{k-1}-\nabla_b f(\hat{\bfu}^{k-1})}{L_k+\lambda_3},\\[0.2cm]
\bfw^k=\frac{1}{L_k+\lambda_2}\mcs_{\lambda_1}\big(L_k\hat{\bfw}^{k-1}-\nabla_\bfw f(\hat{\bfu}^{k-1})\big),
\end{array}\right.
\end{equation}
where $\mcs_\nu(\cdot)$ is a component-wise shrinkage operator defined by $\mcs_\nu(t)=\sign(t)\max(|t|-\nu,0)$.

In \eqref{eq:mainupdate}, we dynamically update $L_k$  by 
\begin{equation}\label{eq:updatel}
L_k=\min(\eta^{n_k}L_{k-1},L_f),
\end{equation} where $\eta>1$ is a pre-selected constant, $L_f$ is defined in \eqref{lipconst} and $n_k$ is the smallest integer such that the following condition is satisfied
\begin{equation}\label{eq:dynamicw}
\begin{array}{ll}
&f(\prox_{h_k}(\bfv^{k-1}))\\[0.1cm]
\le & f(\hat{\bfu}^{k-1})+\left\langle\nabla f(\hat{\bfu}^{k-1}), \prox_{h_k}(\bfv^{k-1})-\hat{\bfu}^{k-1}\right\rangle\\
&\hspace{0.2cm}+\frac{L_k}{2}\left\|\prox_{h_k}(\bfv^{k-1})-\hat{\bfu}^{k-1}\right\|^2 ,
\end{array}
\end{equation}
where $h_k(\bfu)=\frac{1}{L_k}g(\bfu)$, $\bfv^{k-1}=\hat{\bfu}^{k-1}-\frac{1}{L_k}\nabla f(\hat{\bfu}^{k-1})$ and $\hat{\bfu}^{k-1}=\bfu^{k-1}+\omega_{k-1}(\bfu^{k-1}-\bfu^{k-2})$ for some weight $\omega_{k-1}\le1$. The inequality in \eqref{eq:dynamicw} is necessary to make the algorithm convergent (see Lemma \ref{lem:key} in the appendix). It guarantees sufficient decrease, and the setting $L_k=L_f$ will make it satisfied. In the case where $L_f$ becomes unnecessarily large, a smaller $L_k$ can speed up the convergence. To make the overall objective non-increasing, we re-update $\bfu^k$ by setting $\hat{\bfu}^{k-1}=\bfu^{k-1}$ in \eqref{eq:mainupdate} whenever $F(\bfu^k)>F(\bfu^{k-1})$. \emph{This non-increasing monotonicity is very important, since we observe that the PG method may be numerically unstable without this modification. In addition, it significantly accelerates the PG method; see Table \ref{table:diffset} below.} 

Algorithm \ref{alg:mhhsvm}  summarizes our discussion. We name it as B-PGH.

\begin{algorithm}
\caption{Proximal gradient method for B-HSVM (B-PGH)}\label{alg:mhhsvm}
{\small
\begin{algorithmic}[1]
\STATE {\bf Input:} sample-label pairs $(\bfx_i,y_i), i=1,\cdots,n$; parameters $\lambda_1,\lambda_2,\lambda_3, \delta$.
\STATE {\bf Initialization:} choose $\bfu^0=(b^0;\bfw^0)$, $\bfu^{-1}=\bfu^0$; compute $L_f$ from \eqref{lipconst} and choose $\eta>1$ and $L_0\le L_f$; set $k=1$.
\WHILE{\emph{Not converged}}
\STATE Let $\hat{\bfu}^{k-1}=\bfu^{k-1}+\omega_{k-1}(\bfu^{k-1}-\bfu^{k-2})$ for some $\omega_{k-1}\le 1$;
\STATE Update $L_k$ according to \eqref{eq:updatel} and $\bfu^k$ according to \eqref{eq:mainupdate};
\IF{$F(\bfu^k)> F(\bfu^{k-1})$}
\STATE Re-update $\bfu^k$ according to \eqref{eq:mainupdate} with $\hat{\bfu}^{k-1}=\bfu^{k-1}$;
\ENDIF
\STATE Let $k=k+1$.
\ENDWHILE
\end{algorithmic}}
\end{algorithm}

\subsection{PG method for multi-class HSVM}
In this section we generalize the PG method for solving multi-class classification problems. 
 Denote $\mbfu=(\bfb;\mbfw).$ Let \begin{align*}G(\bfb,\mbfw)=&~\lambda_1\|\mbfw\|_1+\frac{\lambda_2}{2}\|\mbfw\|_F^2+\frac{\lambda_3}{2}\|\bfb\|^2,\\ \mcu=&~\{(\bfb,\mbfw):\mbfw\bfe=\mathbf{0},\bfe^\top\bfb=0\}.\end{align*} Then we can write \eqref{eq:multihsvm} as 
\begin{equation}\label{eq:multihsvm1}
\min_{\mbfu\in\mcu} H(\mbfu)\equiv\ell_M(\mbfu)+G(\mbfu)
\end{equation}
Similar to Proposition \ref{prop:f}, we can show that $\nabla_\mbfu \ell_M$ is Lipschitz continuous with constant
\begin{equation}\label{lip:multi}
L_{m}= \frac{J}{n\delta}\sum_{i=1}^n (1+\|\bfx_i\|^2).
\end{equation}
The proof is essentially the same as that of Proposition \ref{prop:f}, and we do not repeat it.

Now we are ready to apply the PG method to \eqref{eq:multihsvm} or equivalently \eqref{eq:multihsvm1}. We update $\mbfu$ iteratively by 
solving the $\bfb$-subproblem
\begin{equation}\label{b-sub}
\begin{array}{ll}
\bfb^k=&\underset{\bfe^\top\bfb=0}{\arg\min} \langle\nabla_{\bfb}\ell_M(\hat{\mbfu}^{k-1}),\bfb-\hat{\bfb}^{k-1}\rangle\\
&\hspace{1.5cm}+\frac{L_k}{2}\|\bfb-\hat{\bfb}^{k-1}\|^2+\frac{\lambda_3}{2}\|\bfb\|^2
\end{array}
\end{equation}
and $\mbfw$-subproblem
\begin{equation}\label{w-sub}
\begin{array}{ll}
\mbfw^k=&\underset{\mbfw\bfe=\mathbf{0}}{\arg\min} \langle\nabla_{\mbfw}\ell_M(\hat{\mbfu}^{k-1}),\mbfw-\hat{\mbfw}^{k-1}\rangle\\
&\hspace{0.5cm}+\frac{L_k}{2}\|\mbfw-\hat{\mbfw}^{k-1}\|^2+\lambda_1\|\mbfw\|_1+\frac{\lambda_2}{2}\|\mbfw\|^2,
\end{array}
\end{equation}
where $L_k$ is chosen in the same way as \eqref{eq:dynamicw}.

Problem \eqref{b-sub} is relatively easy and has a closed form solution. Let $\mbfp=\left[\begin{array}{c}\mbfi\\-\bfe^\top\end{array}\right]\in\mbr^{J\times (J-1)}$ and $\bar{\bfb}\in\mbr^{J-1}$ be the vector consisting of the first $J-1$ components of $\bfb$, where $\mbfi$ denotes the identity matrix. Substituting $\bfb=\mbfp\bar{\bfb}$ to \eqref{b-sub} gives 
the solution 
$$\bar{\bfb}^k=\frac{1}{\lambda_3+L_k}(\mbfp^\top\mbfp)^{-1}\mbfp^\top
\big(L_k\hat{\bfb}^{k-1}-\nabla_{\bfb}\ell_M(\hat{\mbfu}^{k-1})\big).$$
Hence, the update in \eqref{b-sub} can be explicitly written as
\begin{equation}\label{up-b}
\bfb^k=\mbfp\bar{\bfb}^k,
\end{equation}
where $\bar{\bfb}^k$ is defined in the above equation.

Problem \eqref{w-sub} can be further decomposed into $p$ independent small problems. Each of them involves only one row of $\mbfw$ and can be written in the following form 
\begin{equation}\label{p-subs}
\min_{\bfw}\frac{1}{2}\|\bfw-\bfz\|^2+\lambda\|\bfw\|_1,\text{ s.t. }\bfe^\top\bfw=0,
\end{equation}
where $\lambda=\frac{\lambda_1}{L_k+\lambda_2}$ and $\bfz^\top$ is the $i$-th row-vector of the matrix $\frac{L_k}{L_k+\lambda_2}\hat{\mbfw}^{k-1}-\frac{1}{L_k+\lambda_2}\nabla_{\mbfw}\ell_M(\hat{\mbfu}^{k-1})$ for the $i$-th small problem. The coexistence of the non-smooth term $\|\bfw\|_1$ and the equality constraint $\bfe^\top\bfw=0$ makes \eqref{p-subs} a difficult optimization problem to solve. Next we describe a new efficient yet simple method for solving \eqref{p-subs} using its dual problem, defined by
\begin{equation}\label{dual}
\begin{array}{ll}
\underset{\sigma}{\max} \>\> \gamma(\sigma)\equiv&\frac{1}{2}\|\mcs_\lambda(\bfz-\sigma)-\bfz\|^2+\lambda\|\mcs_\lambda(\bfz-\sigma)\|_1\\
&\hspace{1.2cm}+\sigma\bfe^\top\mcs_\lambda(\bfz-\sigma).
\end{array}
\end{equation}

Since \eqref{p-subs} is strongly convex, $\gamma(\sigma)$ is concave and continuously differentiable. Hence, the solution $\sigma^*$ of \eqref{dual} can be found by solving the single-variable equation $\gamma'(\sigma)=0$. It is easy to verify that $\gamma'(z_{\min}-\lambda)>0$ and $\gamma'(z_{\max}+\lambda)<0$, so the solution $\sigma^*$ lies between $z_{\min}-\lambda$ and $z_{\max}+\lambda$, where $z_{\min}$ and $z_{\max}$ respectively denote the minimum and maximum components of $\bfz$. In addition, note that $\mcs_\lambda(\bfz-\sigma)$ is piece-wise linear about $\sigma$, and the breakpoints are at $z_i\pm\lambda$, $i=1,\cdots,J$. Hence $\sigma^*$ must be in $[v_l,v_{l+1}]$ for some $l$, where $\bfv$ is the sorted vector of $(\bfz-\lambda;\bfz+\lambda)$ in the increasing order. Therefore, to solve \eqref{dual}, we first obtain $\bfv$, then search the interval that contains $\sigma^*$ by checking the sign of $\gamma'(\sigma)$ at the breakpoints, and finally solve $\gamma'(\sigma)=0$ within that interval. Algorithm \ref{alg:dual} summarizes our method for solving \eqref{dual}.

\begin{algorithm}\caption{Exact method for solving \eqref{dual}}\label{alg:dual}
{\small
\begin{algorithmic}[1]
\STATE {\bf Input:} $(\bfz,\lambda)$ with $\bfz$ in $J$-dimensional space and $\lambda>0$.
\STATE Let $\bfv=[\bfz-\lambda; \bfz+\lambda]\in\mbr^{2J}$ and sort $\bfv$ as $v_1\le v_2\le\cdots\le v_{2J}$; set $l=J$.
\WHILE{$\gamma'(v_l)\cdot\gamma'(v_{l+1})>0$}
\STATE If $\gamma'(v_l)>0$, set $l=l+1$; else $l=l-1$.
\ENDWHILE
\STATE Solve $\gamma'(\sigma)=0$ within $[v_l,v_{l+1}]$ and output the solution $\sigma^*$.
\end{algorithmic}
}
\end{algorithm}

After determining $\sigma^*$, we can obtain the solution of  \eqref{p-subs} by setting $\bfw^*=\mcs_\lambda(\bfz-\sigma^*)$.
Algorithm \ref{alg:multihsvm} summarizes the main steps of the PG method for efficiently solving \eqref{eq:multihsvm}. We name it as M-PGH.
\begin{algorithm}
\caption{Proximal gradient method for M-HSVM (M-PGH)}\label{alg:multihsvm}
{\small
\begin{algorithmic}[1]
\STATE {\bf Input:} sample-label pairs $(\bfx_i,y_i), i=1,\cdots,n$ with each $y_i\in\{1,\cdots,J\}$; parameters $\lambda_1,\lambda_2,\lambda_3, \delta$.
\STATE {\bf Initialization:} choose $\mbfu^0$, $\mbfu^{-1}=\mbfu^0$; compute $L_m$ in \eqref{lip:multi} and choose $\eta>1$ and $L_0\le L_m$; set $k=1$.
\WHILE{Not converged}
\STATE Let $\hat{\mbfu}^{k-1}=\mbfu^{k-1}+\omega_{k-1}(\mbfu^{k-1}-\bfu^{k-2})$ for some $\omega_{k-1}\le 1$;
\STATE Choose $L_k$ in the same way as \eqref{eq:updatel};
\STATE Update $\bfb^k$ by \eqref{up-b};
\FOR{$i=1,\cdots,p$}
\STATE Set $\bfz$ to the $i$-th row of
\STATE  $\frac{L_k}{L_k+\lambda_2}\hat{\mbfw}^{k-1}-\frac{1}{L_k+\lambda_2}\nabla_{\mbfw}\ell_M(\hat{\mbfu}^{k-1})$; 
\STATE Solve \eqref{dual} by Algorithm \ref{alg:dual} with input $(\bfz,\frac{\lambda_1}{L_k+\lambda_2})$ and let $\sigma^*$ be the output;
\STATE Set the $i$-th row of $\mbfw^k$ to be $\mcs_\lambda(\bfz-\sigma^*)$;
\ENDFOR
\IF{$H(\mbfu^k)> H(\mbfu^{k-1})$}
\STATE Re-update $\bfb^k$ and $\mbfw^k$ according to \eqref{up-b} and \eqref{w-sub} with $\hat{\mbfu}^{k-1}=\mbfu^{k-1}$;
\ENDIF
\STATE Let $k=k+1$.
\ENDWHILE
\end{algorithmic}}
\end{algorithm}

\subsection{Convergence results}
Instead of analyzing the convergence of Algorithms \ref{alg:mhhsvm} and \ref{alg:multihsvm}, we do the analysis of the PG method for \eqref{eq:comp} with general $\xi_1$ and $\xi_2$ and regard Algorithms \ref{alg:mhhsvm} and \ref{alg:multihsvm} as special cases. Throughout our analysis, we assume that $\xi_1$ is a differentiable convex function and $\nabla \xi_1$ is Lipschitz continuous with Lipschitz constant $L$. We also assume that $\xi_2$ is strongly convex\footnote{Without strong convexity of $\xi_2$, we only have sublinear convergence as shown in \cite{BeckTeboulle2009}.} with constant $\mu>0$, namely,
$$\xi_2(\bfu)-\xi_2(\bfv)\ge \langle \bfg_\bfv, \bfu-\bfv\rangle+\frac{\mu}{2}\|\bfu-\bfv\|^2,$$
for any $\bfg_\bfv\in\partial \xi_2(\bfv)$ and $\bfu,\bfv\in\text{dom}(\xi_2)$,
where $\text{dom}(\xi_2)=\{\bfu\in\mbr^m:\xi_2(\bfu)<\infty\}$ denotes the domain of $\xi_2$. 

Similar results have been shown by Nesterov \cite{nesterov2007gradient} and Schmidt {\em et al} \cite{schmidt2011convergence}. However, our analysis fits to more general settings. Specifically, we allow dynamical update of the Lipschitz parameter $L_k$ and an acceptable interval of the parameters $\omega_k$. On the other hand,  \cite{nesterov2007gradient,schmidt2011convergence} either require $L_k$ to be fixed to the Lipschitz constant of $\xi_1$ or require specific values for the $\omega_k$'s. In addition, neither of \cite{nesterov2007gradient,schmidt2011convergence} do the re-update step as in line 7 of Algorithm \ref{alg:mhhsvm} or line 13 of Algorithm \ref{alg:multihsvm}. 

We tested the PG method under different settings on synthetic datasets, generated in the same way as described in section \ref{sec:syn}. Our goal is to demonstrate the practical effect that each setting has on the overall performance of the PG method. Table \ref{table:diffset} summarizes the numerical results, which show that PG method under our settings converges significantly faster than that under the settings of  \cite{nesterov2007gradient,schmidt2011convergence}.

To analyze the PG method under fairly general settings, we use the so-called Kurdyka-{\L}ojasiewicz (KL) inequality, which has been widely applied in non-convex optimization. Our results show that the KL inequality can also be applied and simplify the analysis in convex optimization. Extending the discussion of the KL inequality is beyond the scope of this paper and therefore we refer the interested readers to \cite{xu2013bcd,lojasiewicz1993geometrie,kurdyka1998gradients,bolte2007lojasiewicz,xu-yin-nonconvex} and the references therein. 

\begin{table*}\caption{Performance of the PG method with different settings: {\cite{schmidt2011convergence}} sets $L_k= L_f$ for all $k$; neither of {\cite{nesterov2007gradient,schmidt2011convergence}} makes $F(\bfx_k)$ non-increasing. The data used in these tests is generated in the same way as described in section \ref{sec:syn}, and here we set the correlation parameter $\rho=0$.}\label{table:diffset}
{\footnotesize
\begin{center}
\begin{tabular}{|c||ccc|ccc|ccc|}
\hline
{Problems} & \multicolumn{3}{|c|}{Our settings} & \multicolumn{3}{|c|}{Settings of \cite{nesterov2007gradient}} & \multicolumn{3}{|c|}{Settings of \cite{schmidt2011convergence}}\\\hline
$(n,p,s)$ & \#iter & time & obj. & \#iter & time & obj. & \#iter & time & obj.\\\hline
(3000, 300, 30)      & 34    & 0.06    & 1.0178e-1    & 135 & 0.22   & 1.0178e-1   &     475      & 0.75       & 1.0178e-1    \\
(2000, 20000, 200)    & 91    & 4.34    & 8.0511e-2    & 461 & 21.77 & 8.0511e-2   &     2000    & 99.05    & 8.0511e-2    \\\hline
\end{tabular}
\end{center}}
\end{table*}


%

Our main result is summarized as follows.

\begin{theorem}\label{thm:finitesquare}
Let $\{\bfu^k\}$ be the sequence generated by \eqref{eq:pg} with $L_k\le L$ and $\hat{\bfu}^{k-1}=\bfu^{k-1}+\omega_{k-1}(\bfu^{k-1}-\bfu^{k-2})$ for some $\omega_{k-1}\le\sqrt{\frac{L_{k-1}}{L_{k}}}$ such that
\eqref{eq:condL} holds. In addition, we make $\{\xi(\bfu^k)\}$ nonincreasing by re-updating $\bfu^k$ with $\hat{\bfu}^{k-1}=\bfu^{k-1}$ in \eqref{eq:pg} whenever $\xi(\bfu^k)>\xi(\bfu^{k-1})$. Then 
$\bfu^k$ $R$-linearly converges to the unique minimizer $\bfu^*$ of \eqref{eq:comp}, namely, there exist positive constants $C$ and $\tau<1$ such that
\begin{equation}\label{eq:ratefista}
\|\bfu^k-\bfu^*\|\le C\tau^k,\ \forall k\ge0.
\end{equation}
\end{theorem}

%

The proof of this theorem is given in the Appendix due to its technical complexity. Using the results of Theorem \ref{thm:finitesquare}, we establish the convergence results of Algorithms \ref{alg:mhhsvm} and \ref{alg:multihsvm} in the following corollary.
\begin{corollary}
Let $\{\bfu^k\}$ and $\{\mbfu^k\}$ be the sequences generated by Algorithms \ref{alg:mhhsvm} and \ref{alg:multihsvm} with $\lambda_2,\lambda_3>0$ and $\omega_{k-1}\le\sqrt{\frac{L_{k-1}}{L_{k}}}$. Then $\{\bfu^k\}$ and $\{\mbfu^k\}$ $R$-linearly converge to the unique solutions of \eqref{eq:main} and \eqref{eq:multihsvm}, respectively.
\end{corollary}

\begin{remark}
If one of $\lambda_2$ and $\lambda_3$ vanishes, we only have sublinear convergence by some appropriate choice of $\omega_k$. The results can be found in {\cite{BeckTeboulle2009}}.
\end{remark}


\subsection{Two-stage accelerated method for large-scale problems} \label{sec:2-stage-accelerated-method}
Most cost of Algorithm \ref{alg:mhhsvm} at iteration $k$ occurs in the computation of $\nabla f({\hat{\bfu}^{k-1}})$ and the evaluation of $F(\bfu^k)$. Let $\mbfv=(\bfv_1,\cdots,\bfv_n)^\top$ where $\bfv_1,\cdots,\bfv_n$ are defined in \eqref{notation_uv}. To compute $\nabla f({\hat{\bfu}^{k-1}})$, we need to have $\mbfv\hat{\bfu}^{k-1}$ and $\nabla f_i(\hat{\bfu}^{k-1}), i=1,\cdots,n$, which costs $O(np)$ in total. Evaluating $F(\bfu^k)$ needs $\mbfv\bfu^k$ which costs $O(np)$. To save the computing time, we store both $\mbfv\bfu^{k-1}$ and $\mbfv\bfu^{k}$ so that $\mbfv\hat{\bfu}^k$ can be obtained by $\mbfv\hat{\bfu}^k=\mbfv\bfu^{k}+\omega_k(\mbfv\bfu^{k}-\mbfv\bfu^{k-1})$. This way, we only need to compute $\mbfv\bfu^k$ once during each iteration, and the total computational cost is $O(Tnp)$ where $T$ is the total number of iterations.

As $p$ is large and the solution of \eqref{eq:main} is sparse, namely, only a few features are relevant, we can further reduce the cost of Algorithm \ref{alg:mhhsvm} by switching from the original high-dimensional problem to a lower-dimensional one. More precisely, we first run Algorithm \ref{alg:mhhsvm} with $L_k=L_f$ and $\omega_k=0$ until the support of $\bfw^k$ remains almost unchanged. Then we reduce \eqref{eq:main} to a lower-dimensional problem by confining $\bfw$ to the detected support, namely, all elements out of the detected support are kept \emph{zero}. Finally, Algorithm \ref{alg:mhhsvm} is employed again to solve the lower-dimensional problem. The two-stage method for \eqref{eq:main} is named as B-PGH-2, and its solidness rests on the following lemma, which can be shown in a similar way as the proof of Lemma 5.2 in \cite{HaleYinZhang2007}.

\begin{lemma}
Let $\{(b^k,\bfw^k)\}$ be the sequence generated by Algorithm \ref{alg:mhhsvm} with $L_k=L_f$ and $\omega_k=0$ starting from any $\bfu^0=(b^0,\bfw^0)$. Suppose $\bfu^*=(b^*,\bfw^*)$ is the unique solution of \eqref{eq:main} with $\lambda_3>0$. Let $$\begin{array}{l}Q(b,\bfw)=f(b,\bfw)+\frac{\lambda_2}{2}\|\bfw\|^2+\frac{\lambda_3}{2}b^2,\\[0.1cm]
 h_i(\bfw)=w_i-\frac{1}{L_f+\lambda_2}\nabla_{w_i}Q(\bfu),\end{array}$$
and
$$\begin{array}{l}\mci=\{i:|\nabla_{w_i}Q(\bfu^*)|<\lambda_1\},\
 \mce=\{i:|\nabla_{w_i}Q(\bfu^*)|=\lambda_1\},\end{array}$$
where $w_i$ is the $i$th component of $\bfw$.
Then $\supp(\bfw^*)\subset\mce$ and $w^*_i=0, \forall i\in\mci$, where $\supp(\bfw^*)$ denotes the support of $\bfw^*$. In addition, $w_i^k=0,\forall i\in \mci$ and $\sign(h_i(\bfw^k))=\sign(h_i(\bfw^*)), \forall i\in \mce$ for all but at most finite iterations.
\end{lemma}

Suppose the cardinality of the solution support is $s$. Then the total computational cost of the two-stage method B-PGH-2 is $O(T_1np+T_2ns)$, where $T_1,T_2$ are the numbers of iterations in the first and second stages, respectively. Numerically, we found that $\supp(\bfw^*)$ could be detected in several iterations, namely, $T_1$ is usually small. When $s\ll p$, B-PGH-2 can be significantly faster than B-PGH, as demonstrated by our experiments in Section \ref{sec:numerical}. In the same way, Algorithm \ref{alg:multihsvm} can be accelerated by a two-stage method. We omit the analysis since it can be derived by following the same steps.

\section{Numerical Experiments}\label{sec:numerical}
In the first part of this section, we compare  B-PGH, described in Algorithm \ref{alg:mhhsvm}, with two very recent binary-class SVM solvers using ADMM \cite{ye2011efficient} and GCD\footnote{Our algorithm and ADMM are both implemented in MATLAB, while the code of GCD is written in R. To be fair, we re-wrote the code of GCD and implemented it in MATLAB.} \cite{yangefficient} on both synthetic and real data. ADMM solves model \eqref{eq:elastic} whereas both B-PGH and GCD solve model  \eqref{eq:main}. 
In the second part, we compare the performance of five different multi-class SVMs which are: the model defined by \eqref{eq:multihsvm}, the $\ell_1$-regularized M-SVM in \cite{wang20071}, the $\ell_\infty$-regularized M-SVM in \cite{zhang2008variable}, the ``one-vs-all'' (OVA) method \cite{bottou1994comparison}, and the Decision Directed Acyclic Graph (DDAG) method \cite{platt2000large}. We use M-PGH\footnote{The paper \cite{li2010huberized} uses a path-following method to solve \eqref{eq:multihsvm}. However, its code is not publicly available.}, described in Algorithm \ref{alg:multihsvm}, to solve \eqref{eq:multihsvm} and Sedumi \cite{sturm1999using} for the $\ell_1$ and $\ell_\infty$-regularized M-SVMs. Sedumi is called through CVX \cite{grant2008cvx}. All the tests were performed on a computer having an i7-620m CPU and 3-GB RAM and running 32-bit Windows 7 and Matlab 2010b.

\subsection{Binary-class SVM}
For B-PGH, the parameters $\eta$ and $L_0$ were set to $\eta=1.5$ and $L_0=\frac{2L_f}{n}$, respectively. We chose $$\omega_{k-1}=\min\left(\frac{t_{k-1}-1}{t_k},\sqrt{\frac{L_{k-1}}{L_k}}\right),$$ where $t_0=1$ and $t_k=\frac{1}{2}\left(1+\sqrt{1+4t_{k-1}^2}\right)$. The starting point was chosen to be a \emph{zero} vector. We stop B-PGH if both of the following conditions are satisfied during three consecutive iterations
\begin{equation}\label{eq:stop}
\frac{F_{k-1}-F_k}{1+F_{k-1}}\le tol,\qquad
\frac{\|\bfu^{k-1}-\bfu^k\|}{1+\|\bfu^{k-1}\|}\le tol,
\end{equation}
where $\bfu^k=(b^k;\bfw^k)$ and $F_k=F(\bfu^k)$. The stopping tolerance was set to $tol=10^{-6}$ for B-PGH and GCD and $tol=10^{-5}$ for ADMM since $tol=10^{-6}$ was too stringent. For B-PGH-2, we took $tol=10^{-3}$ at the first stage and $tol = 10^{-6}$ at the second stage. The penalty parameters for ADMM were set to $\mu_1=\frac{100}{n}$ and $\mu_2=50$ as suggested in \cite{ye2011efficient}.
All the other parameters for ADMM and GCD were set to their default values. 

\subsubsection{Synthetic data}\label{sec:syn}
We generated $n$ samples in $\mbr^p$ with one half in the ``$+1$'' class and the other half in the ``$-1$'' class. Specifically, each sample in the ``$+1$'' class was generated according to the Gaussian distribution $\mathcal{N}(\bm{\mu},\mathbf{\Sigma})$ and each sample in the ``$-1$'' class according to $\mathcal{N}(-\bm{\mu},\mathbf{\Sigma})$. The mean vector 
$$\bm{\mu}=(\underset{s}{\underbrace{1,\cdots,1}},\underset{p-s}{\underbrace{0,\cdots,0}})^\top,$$
and the covariance matrix
$$\mathbf{\Sigma}=\begin{bmatrix}
\rho\mathbf{1}_{s\times s}+(1-\rho)\mathbf{I}_{s\times s} & \mathbf{0}_{s\times (p-s)}\\
\mathbf{0}_{(p-s)\times s} & \mathbf{I}_{(p-s)\times (p-s)}
\end{bmatrix},$$
where $\mathbf{1}_{s\times s}$ is the matrix of size $s\times s$ with all \emph{one}'s, $\mathbf{0}$ is the matrix with all \emph{zero}'s, and $\mathbf{I}_{s\times s}$ is an identity matrix of size $s\times s$. The first $s$ variables are relevant for classification and the rest ones being noise. This simulated data was also used in \cite{wang2008hybrid,ye2011efficient}. We tested the speed of the algorithms on different sets of dimension $(n,p,s)$ and the correlation $\rho$.  The smoothing parameter for B-PGH and GCD was set to $\delta=1$ in this subsection, and $\lambda_3=\lambda_2$ is set in \eqref{eq:main}. Table \ref{table:syn_time} lists the average running time (sec) of 10 independent trials. For each run, we averaged the time over 25 different pairs of $(\lambda_1,\lambda_2)$. From the results, we can see that B-PGH was consistently faster than GCD and over 50 times faster than ADMM. In addition, the two-stage accelerated algorithm B-PGH-2 (see section \ref{sec:2-stage-accelerated-method}) was significantly faster than B-PGH when $p$ was large and $s\ll p$.

\begin{table*}
\caption{Running time (in seconds) of {B-PGH}, {GCD} and {ADMM}. Each result is the average of 10 independent trials$^*$. For each run, the time is averaged over 25 different pairs of $(\lambda_1,\lambda_2)$.}
\label{table:syn_time}
{\small $^*$ For $n=2000,p=20000$, the time for ADMM is over one half hour.}\vspace{0.1cm}
{\footnotesize
\begin{center}
\begin{tabular}{|c||cc|cc|cc|cc|}\hline
{Problems} & \multicolumn{2}{|c|}{B-PGH} & \multicolumn{2}{|c|}{B-PGH-2} & \multicolumn{2}{|c|}{GCD} & \multicolumn{2}{|c|}{ADMM}\\\hline\hline
$(n,p,s)$  & $\rho=0$ & $\rho=0.8$ & $\rho=0$ & $\rho=0.8$ & $\rho=0$ & $\rho=0.8$ & $\rho=0$ & $\rho=0.8$\\\hline
(500, 50, 5)  & 0.0060 & 0.0081 & 0.0059 & 0.0071 & 0.0116 & 0.0192 & 0.6354 & 0.8641\\\hline
(2000, 100, 10)  & 0.0176 & 0.0256 & 0.0173 & 0.0218 & 0.0848 & 0.1469 & 2.4268 & 3.5340\\\hline
(50, 300, 30)  & 0.0126 & 0.0162 & 0.0099 & 0.0113 & 0.0338 & 0.0409 & 0.6819 & 1.1729\\\hline
(100, 500, 50)  & 0.0179 & 0.0242 & 0.0117 & 0.0152 & 0.0727 & 0.0808 & 1.4879 & 2.5482\\\hline
(200, 2000, 100) & 0.0720 & 0.1227 & 0.0301 & 0.0446 & 0.5653 & 0.4735 & 7.9985 & 12.998\\\hline
(2000, 20000, 200)  & 5.7341 & 8.5379 & 1.1543 & 1.7531 & 32.721 & 30.558 & ------ & ------\\\hline
\end{tabular}
\end{center}}
\end{table*}

We also tested the prediction accuracy and variable selection of the algorithms. The problem dimension was set to $n=50,p=300,s=20$. The optimal pair of $(\lambda_1,\lambda_2)$ was selected from a large grid by 10-fold cross validation. 
We use $n_t$ for the number of selected relevant variables and $n_f$ for the number of selected noise variables. In addition, we use ``accu.'' for the prediction accuracy. The average results of 500 independent runs corresponding to $\rho=0$ and $\rho=0.8$ are shown in Table \ref{table:syn_accu}. During each run, the algorithms were compared on a test set of 1000 samples. From the table, we can see that B-PGH achieved similar accuracy to that by GCD, and they both performed better than ADMM, especially in the case of $\rho=0.8$. In addition, B-PGH-2 obtained similar solutions as B-PGH.

\begin{table*}
\caption{Classification accuracies and variable selections of {B-PGH}, {GCD} and {ADMM}. All results are the averages of 500 independent runs, each of which tests on a 1000 tesing set. The numbers in the parentheses are the corresponding standard errors.}
\label{table:syn_accu}
{\footnotesize
\begin{center}
\begin{tabular}{|c||ccc|ccc|}\hline
\multirow{2}{*}{Algorithms}&\multicolumn{3}{|c|}{$\rho=0$} & \multicolumn{3}{|c|}{$\rho=0.8$}\\
 & $n_t$ & $n_f$ & accu.(\%) & $n_t$ & $n_f$ & accu.(\%)\\\hline
B-PGH & 20.0(0.1) & 0.1(0.4) & 100(0.000) & 19.9(0.3) & 7.3(3.9) & 86.6(0.011)\\
B-PGH-2 & 20.0(0.1) & 0.1(0.3) & 100(0.000) & 19.9(0.4) & 7.5(4.1) & 86.6(0.011)\\
GCD & 20.0(0.1) & 0.1(0.3) & 100(0.000) & 19.9(0.4) & 7.4(3.8) & 86.4(0.011)\\
ADMM & 19.0(0.9) & 2.9(1.8) & 100(0.000) & 17.2(1.5) & 23.1(6.0) & 85.7(0.013)\\\hline
\end{tabular}
\end{center}}
\end{table*}

\subsubsection{Medium-scale real data}\label{sec:realdata}
In this subsection, we compare B-PGH, GCD and ADMM on medium-scale real data (see Table \ref{table:data}). The first seven datasets are available from the LIBSVM dataset\footnote{http://www.csie.ntu.edu.tw/$\sim$cjlin/libsvmtools/datasets} and the last three from Tom Mitchell’s neuroinformatics research group\footnote{http://www.cs.cmu.edu/$\sim$tom/fmri.html}. Both {\bf rcv1} and {\bf realsim} have large feature dimensions but only about $0.2\%$ nonzeros. {\bf colon}, {\bf duke} and {\bf leuk} are datasets of gene expression profiles for colon cancer, breast cancer and leukemia, respectively. The original dataset of {\bf colon} consists of 62 samples, and we randomly chose 30 of them for training and the rest for tesing. {\bf gisette} is a hand-writing digit recognition problem from NIPS 2003 Feature Selection Challenge. The training set for {\bf gisette} is a random subset of the original 6000 samples, and the testing set contains all of the original 1000 samples. {\bf rcv1} is a collection of manually cetegorized news wires from Reuters. Both the training and tesing instances for {\bf sub-rcv1} are randomly sub-sampled from the original training and tesing samples. {\bf realsim} contains UseNet articles from four discussion groups, for simulated auto racing, simulated aviation, real autos and real aviation. The original dataset of {\bf realsim} has 72,309 samples, and we randomly sub-sampled 500 instances for training and 1,000 instances for tesing. {\bf fMRIa}, {\bf fMRIb}, and {\bf fMRIc} are functional MRI (fMRI) data of brain activities when the
subjects are presented with pictures and text paragraphs.

\begin{table}
\caption{The size and data type of the tested real datasets}
\label{table:data}
{\footnotesize
\begin{center}
\resizebox{0.45\textwidth}{!}{\begin{tabular}{|c|c|c|c|c|}\hline
Dataset & \#training & \#testing & \#features & Type\\\hline
{\bf australian} & 200 & 490 & 14 & Dense\\
{\bf colon} & 30 & 32 & 2,000 & Dense \\
{\bf duke} & 32 & 10 & 7,129 & Dense \\
{\bf gisette} & 500 & 1,000 & 5,000 & Dense \\
{\bf leuk} & 38 & 34 & 7,129 & Dense \\
{\bf sub-rcv1} & 500 & 1,000 & 47,236 & Sparse \\
{\bf sub-realsim} & 500 & 1,000 & 20,958 & Sparse \\
{\bf fMRIa} & 30 & 10 & 1715 & Dense\\
{\bf fMRIb} & 30 & 10 & 1874 & Dense\\
{\bf fMRIc} & 30 & 10 & 1888 & Dense\\\hline
\end{tabular}}
\end{center}}
\end{table}

We fixed $\lambda_2=\lambda_3=1$ since the algorithms appeared not sensitive to $\lambda_2$ and $\lambda_3$. The optimal $\lambda_1$'s were tuned by 10-fold cross-validation on training sets. The smoothing parameter was set to $\delta=1$ for both B-PGH and GCD. All the other settings are the same as those in the previous test. The results are shown in Table \ref{table:real}. For comparison, we also report the prediction accuracies by LIBLINEAR \cite{fan2008liblinear} with $L_1$-regularized $L_2$-loss. From the results, we can see that B-PGH is significantly faster than GCD and ADMM, and it also gives the best prediction accuracy. B-PGH-2 was fastest and achieved the same accuracy as B-PGH except for {\bf gisette}. We want to mention that GCD can give the same accuracy as B-PGH but it needs to run much longer time, and ADMM can rarely achieve the same accuracy even if it runs longer since it solves a different model. 

\begin{table*}
\caption{Comparison results of {B-PGH}, {GCD}, {ADMM} and LIBLINEAR on real data. The best prediction accuracy for each dataset is highlighted in {\bf bold} and the best running time (sec) in {\it italics}.}
\label{table:real}
{\footnotesize
\begin{center}
\resizebox{\textwidth}{!}{\begin{tabular}{|c||ccc|ccc|ccc|ccc|c|}\hline
\multirow{2}{*}{Dataset} & \multicolumn{3}{|c|}{B-PGH} & \multicolumn{3}{|c|}{B-PGH-2} &  \multicolumn{3}{|c|}{GCD} & \multicolumn{3}{|c|}{ADMM}&LIBLINEAR\\\cline{2-14}
 & accu(\%) & supp & time & accu(\%) & supp & time & accu(\%) & supp & time & accu(\%) & supp & time&accu(\%)\\\hline
{\bf australian} &{\bf 87.4} & 11 & {\it 0.01} &{\bf 87.4} & 11 & {\it 0.01} &86.7 & 10 & 0.02 &86.9 & 14 & 1.08 & 85.7\\
{\bf colon}  & {\bf 84.4} & 89 & {\it 0.04} & {\bf 84.4} & 89 & 0.05 & {\bf 84.4} & 89 & 0.38 & {\bf 84.4} & 118 & 1.48 & 81.3\\
{\bf duke} & {\bf 90} & 118 & 0.20 & {\bf 90} & 118 & {\it 0.10} & {\bf 90} & 112 & 0.93 & {\bf 90} & 171 & 3.11 & 80\\
{\bf gisette}  & {\bf 92.9} & 977 & 1.99 & 92.7 & 946 & {\it 1.94} & 92.6 & 959 & 17.61 & 92.8 & 1464 & 218.5 &91.7\\
{\bf leuk}  & {\bf 91.2} & 847 & 0.19 & {\bf 91.2} & 846 & {\it 0.15} & 82.4 & 716 & 3.10 & 82.4 & 998 & 2.35 & {\bf 91.2}\\
{\bf sub-rcv1}  & {\bf 84.8} & 1035 & 0.06 & {\bf 84.8} & 1035 & {\it 0.05} & 83.3 & 1035 & 3.46 & 82.3 & 1776 & 2.61 &80.1\\
{\bf sub-realsim}  & {\bf 92.9} & 1134 & 0.04 & {\bf 92.9} & 1134 & {\it 0.02} & 91.9 & 1134 & 2.96 & 92.8 & 1727 & 1.61 & 90.9\\
{\bf fMRIa} & {\bf 90} & 141 & 0.07 & {\bf 90} & 141 & {\it 0.06} & 70 & 130 & 5.31 & 70 & 203 & 0.57 & 70 \\
{\bf fMRIb} & {\bf 100} & 1098 & 0.11 & {\bf 100} & 1108 & 0.07 & 90 & 180 & 2.26 & {\bf 100} & 1767 & {\it 0.03} & 70 \\
{\bf fMRIc} & {\bf 100} & 1827 & 0.10 & {\bf 100} & 1825 & 0.08 & 80 & 1324 & 2.05 & 90 & 1882 & {\it 0.06} & 50 \\\hline
\end{tabular}}
\end{center}}
\end{table*}

\subsubsection{Statistical comparison}\label{sec:stat-real}
We also performed the statistical comparison of {B-PGH}, GCD, and {ADMM}. Following \cite{demvsar2006statistical}, we did the Wilcoxon signed-ranks test\footnote{The Wilcoxon signed-ranks test is in general better than the paired $t$-test as demonstrated in \cite{demvsar2006statistical}.} \cite{wilcoxon1945individual} and Friedman test \cite{friedman1937use,friedman1940comparison} to see if the differences of the compared methods are significant. The former test is for pair-wise comparison and the latter one for multiple comparison. Specifically, 
for two different methods, denote $d_i$ as the difference of their score (e.g., prediction accuracy) on the $i$-th dataset and rank $d_i$'s based on their absolute value. Let
\begin{equation}\label{eq:z-value}
z=\frac{T-\frac{1}{4}N(N+1)}{\sqrt{\frac{1}{24}N(N+1)(2N+1)}},
\end{equation}
where $N$ is the number of datasets, $T=\min(R^+,R^-)$, and
\begin{align*}
&R^+=\sum_{d_i>0}\text{rank}(d_i)+\frac{1}{2}\sum_{d_i=0}\text{rank}(d_i),\\
&R^-=\sum_{d_i<0}\text{rank}(d_i)+\frac{1}{2}\sum_{d_i=0}\text{rank}(d_i).
\end{align*}
Table \ref{table:wilcoxon-real} shows the pair-wise $z$-values and the corresponding $p$-values of the five compared methods. At $p$-$\text{value}<0.05$, we find that there is no significant differences between B-PGH and B-PGH-2 and neither between GCD and ADMM, and any other pair of methods make significant difference.

The Friedman statistic can be calculated by 
\begin{equation}\label{eq:friedman}
\chi_F^2=\frac{12N}{K(K+1)}\left[\sum_{j=1}^K AR_j-\frac{K(K+1)^2}{4}\right],
\end{equation}
where $K$ is the number of compared methods, $AR_j=\sum_{i=1}^N r_i^j$ denotes the average rank of the $j$-th method, and $r_i^j$ is the rank of the $j$-th method on the $i$-th dataset. Table \ref{table:rank-real} shows the ranks of the compared methods according to their prediction accuracies, where average ranks are applied to ties. The $p$-value indicates significant difference among the five methods at $p$-$\text{value}<0.05$.

We further proceeded with a post-hoc test through Holm's step-down procedure \cite{holm1979simple}. Assuming the $p$-values for compared classifiers to the control one are ordered as $p_1\le p_2\le\ldots\le p_{K-1}$, Holm's step-down procedure starts from the largest one and compares it to $\frac{\alpha}{(K-1)}$, where $\alpha$ is the target $p$-value. If $p_1<\frac{\alpha}{(K-1)}$, it rejects the corresponding null-hypothesis and  compares $p_2$ to $\frac{\alpha}{(K-2)}$, and so on. We set B-PGH as the control classifier and used $(R_1-R_j)/\sqrt{\frac{K(K+1)}{6N}}$ to find the $p$-value for the $j$-th method compared to the control one. The $p$-values are 0.8875, 0.0072, 0.0477, and $5.57\times 10^{-5}$ respectively for B-PGH-2, GCD, ADMM, and LIBLINEAR. Hence, at $\alpha=0.05$,  B-PGH made significant difference with GCD and LIBLINEAR but not with B-PGH-2 or ADMM, and at  $\alpha=0.10$, B-PGH made significant difference with all other methods except B-PGH-2.

\begin{table}\caption{$z$-value (above diagonal) and $p$-value (below diagnal) of Wilcoxon signed-ranks test of {B-PGH}, {GCD}, {ADMM} and LIBLINEAR on real datasets.}\label{table:wilcoxon-real}
\centering
\resizebox{0.45\textwidth}{!}{\begin{tabular}{|c|ccccc|}
\hline
Methods & {B-PGH} & B-PGH-2 & {GCD} & {ADMM} & LIBLINEAR\\\hline
B-PGH & -- &  -0.5096 &  -2.6502 &  -2.4973 &  -2.7521\\
B-PGH-2   & 0.6102   &    -- &  -2.6502 &  -2.0896 &  -2.7521\\
 GCD  & 0.0080  &  0.0080    &     --  & -1.4780 &  -2.0386\\
 ADMM &  0.0126 &   0.0366   & 0.1394  &       --  & -2.0386\\
 LIBLINEAR  & 0.0060  &  0.0060 &   0.0414  &  0.0414        & --\\\hline
\end{tabular}}
\end{table} 

\begin{table}\caption{Friedman ranking of {B-PGH}, {GCD}, {ADMM} and LIBLINEAR according to their prediction accuaries on real datasets.}\label{table:rank-real}
\centering
\resizebox{0.45\textwidth}{!}{\begin{tabular}{|c|ccccc|}
\hline
Dataset & {B-PGH} & B-PGH-2 & {GCD} & {ADMM} & LBLINEAR\\\hline
\textbf{australian} & 1.5 & 1.5 & 4 & 3 & 5\\ 
\textbf{colon} & 2.5 & 2.5 & 2.5 & 2.5 & 5\\ 
\textbf{duke} & 2.5 & 2.5 & 2.5 & 2.5 & 5\\
\textbf{gisette} & 1 & 2 & 4 & 3 & 5 \\
\textbf{leuk} & 2 & 2 & 4.5 & 4.5 & 2\\
\textbf{sub-rcv1} & 1.5 & 1.5 & 3 & 4 & 5\\
\textbf{sub-realsim} & 1.5 & 1.5 & 4 & 3 & 5\\
\textbf{fMRIa} & 1.5 & 1.5 & 4 &  4 & 4 \\
\textbf{fMRIb} & 2 &   2 &  4 &  2 &  5 \\
\textbf{fMRIc}  & 1.5 & 1.5 & 4  & 3 & 5 \\\hline\hline
Average rank & 1.75 & 1.85 & 3.65 & 3.15 & 4.6\\\hline
 \multicolumn{6}{|c|}{$\chi_F^2$-$\text{value}= 23.56$, $p$-$\text{value}=9.78\times 10^{-5}$}\\\hline
\end{tabular}}
\end{table}

\subsubsection{Large-scale real data}
This subsection compares B-PGH, GCD and ADMM on two large-scale datasets. The first one is {\bf rcv1}, part of which has been tested in section \ref{sec:realdata}. It has 20,242 training samples, and each sample has 47,236 features. We use the same 1,000 samples as in section \ref{sec:realdata} for testing. The second dataset contains all the 72,309 samples of {\bf realsim}, part of which is used in section \ref{sec:realdata}, and each sample has 20,958 features. We randomly selected 20,000 samples for testing and the rest for training. Both the two datasets are highly sparse. For all algorithms, we fixed $b=0$ since we observed that using the bias term would affect the prediction accuracy. For B-PGH and GCD, we set $\delta=1$. The best values of $\lambda_1$ and $\lambda_2$ were searched from a large grid by 10-fold cross-validation. The parameters for B-PGH and GCD were set the same as above. ADMM suffered from memory problem since it needs to explicitly form the matrix $\mbfx^\top\mbfx$, which is dense even though $\mbfx$ is sparse, where the $i$-th column of $\mbfx$ was formed by the $i$-th data point $\bfx_i$. Hence, we did not report the results of ADMM. The results by B-PGH and GCD are shown in Table \ref{table:large-real}, where we also reported the prediction accuracy by LIBLINEAR for comparison. From the results we can conclude that both or our algorithms, B-PGH and B-PGH-2, are significantly faster (almost 400 times) than the GCD method. In addition, the accuracy of B-PGH and B-PGH-2 is very similar to that of GCD. 

\begin{table*}
\caption{Comparison of the accuracy and time (sec) for {B-PGH} and {GCD} on the large-scale real datasets {\bf rcv1} and {\bf realsim}.}
\label{table:large-real}
{\footnotesize
\begin{center}
\begin{tabular}{|c||ccc|ccc|ccc|c|}\hline
\multirow{2}{*}{Dataset} & \multicolumn{3}{|c|}{B-PGH} &  \multicolumn{3}{|c|}{B-PGH-2} & \multicolumn{3}{|c|}{GCD}&LIBLINEAR\\\cline{2-11}
 & accu(\%) &supp. & time & accu(\%) &supp. & time & accu(\%) &supp. & time & accu(\%)\\\hline
{\bf rcv1} & \bf{100} &2188& 9.72 & \bf{100} &2159 & {\it 9.54} & 99.7&4253 & 8384.57 & 99.5\\
{\bf realsim} & 96.7&3506 & {\it 16.81} & 96.7 &3429& 18.75 & 96.9&5092 & 8028.62 & \bf{97.0}\\\hline
\end{tabular}
\end{center}}
\end{table*}

\subsubsection{Effects of the smoothing parameter $\delta$}
We tested how $\delta$ affected B-PGH and GCD on the real datasets used in section \ref{sec:realdata}. Since the cost reduction by B-PGH-2 was not significant for these datasets as shown in Table \ref{table:real}, we did not include it in this test. All parameters were set to the same values as in section \ref{sec:realdata} except for $\delta$, which varied between $0.1$ and $0.01$. The running time and prediction accuracies are shown in Table \ref{table:delta}. Comparing the results with those in Table \ref{table:real}, we find that the algorithms tend to give more accurate predictions. However, the accuracy corresponding to $\delta=0.01$ is hardly improved over $\delta=0.1$. In addition, the solutions have more nonzeros, i.e., more features are selected. For these reasons, we do not recommend to choose very small $\delta$. We observed that $\delta\in[0.1,1]$ was fine in all our tests. Furthermore, comparing the columns that show the time in Tables  \ref{table:real} and \ref{table:delta} we observe that the efficiency of the GCD method was greatly affected by different values of $\delta$. In most cases, GCD can become significantly slow with small $\delta$'s. On the contrary, the efficiency of B-PGH was seldom affected by different values of $\delta$. 

\begin{table*}
\caption{Performance of {B-PGH} and {GCD} on real data with different values of the smoothing parameter ($\delta=0.1,0.01$) in the huberized loss function $\phi_H$.}
\label{table:delta}
{\footnotesize
\begin{center}
\begin{tabular}{|c||ccc|ccc||ccc|ccc|}\hline
\multirow{4}{*}{Dataset} & \multicolumn{6}{|c||}{$\delta = 0.1$} & \multicolumn{6}{|c|}{$\delta = 0.01$}\\\cline{2-13}
  & \multicolumn{3}{|c|}{B-PGH} &  \multicolumn{3}{|c||}{GCD}  & \multicolumn{3}{|c|}{B-PGH} &  \multicolumn{3}{|c|}{GCD}\\\cline{2-13}
 & accu(\%) &supp & time & accu(\%) &supp & time & accu(\%) &supp & time & accu(\%) &supp & time\\\hline
{\bf australian} & 85.3 & 11 & {\it 0.01} & 85.3 & 11 & 0.12 & 85.5 & 11 & {\it 0.01} & 86.1 & 11 & 0.24\\
{\bf colon} & 84.4& 109 & {\it 0.07} & 84.4 &106 & 1.16 & 84.4 &123 & {\it 0.14} & 84.4 &118 & 4.98\\
{\bf duke} & 90&147 & {\it 0.30} & 90 &158 & 4.98 & 90&344 & {\it 0.40} & 90 &181 & 10.3\\
{\bf gisette} & 93.1 & 1394 & {\it 2.74} & 93.2&1525 & 4.83  & 93.0&2781 & {\it 3.25} & 92.7&1788 & 19.2\\
{\bf leuk}  & 91.2 & 1006 & {\it 0.57} & 88.2&748 & 18.5  & 91.2&3678 & {\it 0.65} & 91.2 & 970& 18.5\\
{\bf sub-rcv1}  & 85.1 &1040 & {\it 0.03} & 85.1&1040 & 4.71  & 85.1&1040 & {\it 1.16} & 85.1&1040 & 10.7\\
{\bf sub-realsim}  & 93.2 &1145& {\it 0.02} & 93.2&1145 & 3.07  & 93.2&1145 & {\it 0.03} & 93.2 &1145& 8.92\\
{\bf fMRIa} & 90 & 156 & {\it 0.18} & 80 & 149 & 28.35& 90 & 237 & {\it 0.28} & 90 & 217 & 28.35\\
{\bf fMRIb} & 100 & 1335 & {\it 0.26} & 80 & 386 & 2.27& 100 & 1874 & {\it 0.23} & 100 & 1399 & 1.48\\
{\bf fMRIc} & 100 & 1856 & {\it 0.21} & 100 & 1269 & 2.19& 100 & 1888 & {\it 0.24} & 100 & 1888 & 2.56\\\hline
\end{tabular}
\end{center}}
\end{table*}

\subsection{Multi-class SVM} This subsection tests the performance of M-PGH for solving \eqref{eq:multihsvm} on a synthetic dataset, eight benchmark datasets, and also two microarray datasets. 
The parameters of M-PGH were set in the same way as those of B-PGH except we set $L_0=\frac{L_m}{nJ}$, where $L_m$ is given in \eqref{lip:multi}. For all the tests, $\delta=1$ was set in \eqref{eq:multihsvm}. 

\subsubsection{Synthetic data}
We compare model \eqref{eq:multihsvm} solved using M-PGH, the $\ell_1$-regularized M-SVM \cite{wang20071}, and the $\ell_\infty$-regularized M-SVM \cite{zhang2008variable} on a four-class example with each sample in $p$-dimensional space. The data in class $j$ was generated from the mixture of Gaussian distributions $\mcn(\bm{\mu}_j,\bm{\Sigma}_j), j = 1,2,3,4$. The mean vectors and covariance matrices are $\bm{\mu}_2=-\bm{\mu}_1, \bm{\mu}_4=-\bm{\mu}_3, \bm{\Sigma}_2=\bm{\Sigma}_1, \bm{\Sigma}_4=\bm{\Sigma}_3$, which take the form of
{\begin{align*}\bm{\mu}_1&=(\underset{s}{\underbrace{1,\cdots,1}},\underset{p-s}{\underbrace{0,\cdots,0}})^\top,\\ 
 \bm{\mu}_3&=(\underset{s/2}{\underbrace{0,\cdots,0}},\underset{s}{\underbrace{1,\cdots,1}},\underset{p-3s/2}{\underbrace{0,\cdots,0}})^\top,\\
\bm{\Sigma}_1&=\begin{bmatrix}
\rho\mathbf{1}_{s\times s}+(1-\rho)\mbfi_{s\times s} & \mathbf{0}_{s\times(p-s)}\\
\mathbf{0}_{(p-s)\times s}& \mbfi_{(p-s)\times (p-s)}\end{bmatrix},\\ \bm{\Sigma}_3&=\begin{bmatrix}
\mbfi_{\frac{s}{2}\times \frac{s}{2}} & \mathbf{0}_{\frac{s}{2}\times s}& \mathbf{0}_{\frac{s}{2}\times (p-\frac{3s}{2})} \\
\mathbf{0}_{s\times \frac{s}{2}} & \rho\mathbf{1}_{s\times s}+(1-\rho)\mbfi_{s\times s} & \mathbf{0}_{s\times (p-\frac{3s}{2})}\\
\mathbf{0}_{(p-\frac{3s}{2})\times \frac{s}{2}}& \mathbf{0}_{(p-\frac{3s}{2})\times s} & \mbfi_{(p-\frac{3s}{2})\times (p-\frac{3s}{2})}
\end{bmatrix}.\end{align*}}This kind of data has been tested in section \ref{sec:syn} for binary classifications. We took $p=500, s=30$ and $\rho=0, 0.8$ in this test. The best parameters for all three models were tuned by first generating $100$ training samples and another $100$ validation samples. Then we compared the three different models with the selected parameters on $100$ randomly generated training samples and $20,000$ random testing samples. The comparison was independently repeated 100 times. The performance of different models and algorithms were measured by prediction accuracy, running time (sec), the number of incorrect zeros (IZ), the number of nonzeros in each column (NZ1, NZ2, NZ3, NZ4). 

Table \ref{table:multisyn} summarizes the average results, where we can see that M-PGH is very efficient in solving \eqref{eq:multihsvm}. In addition we observe that \eqref{eq:multihsvm} tends to give the best predictions. 

\begin{table*}\caption{Results of different models solved by M-PGH and Sedumi on a four-class example with synthetic data. The numbers in the parentheses are corresponding standard errors. Highest predictions are highlighted.}
\label{table:multisyn}
{\footnotesize
\begin{center}
\begin{tabular}{|c|ccccccc|}\hline
 \multirow{2}{*}{models}&accu.(\%) & time  & IZ & NZ1 & NZ2 & NZ3 & NZ4  \\\cline{2-8}
& \multicolumn{7}{|c|}{$\rho=0$} \\\hline
\eqref{eq:multihsvm} by M-PGH & \textbf{96.7}(0.007) & {\it 0.017} & 29.43 & 28.59 & 29.19 & 28.78 & 29.14 \\
$\ell_1$-regularized \cite{wang20071} by Sedumi & 83.0(0.018) & 3.56 & 59.6 & 29.3 & 28.7 & 29.3 & 28.9 \\
$\ell_\infty$-regularized \cite{zhang2008variable}  by Sedumi  & 84.0(0.019) & 20.46 & 33.2 & 49.3 & 49.3 & 49.3 & 49.3 
\\\hline\hline
& \multicolumn{7}{|c|}{$\rho=0.8$}\\\hline
\eqref{eq:multihsvm} by M-PGH & \textbf{78.4}(0.020) & {\it 0.021} & 35.15 & 26.93 & 27.29 & 26.41 & 26.75\\
$\ell_1$-regularized \cite{wang20071} by Sedumi &  64.8(0.024) & 3.50 & 91.6 & 16.4 & 17.1 & 17.2 & 16.4\\
$\ell_\infty$-regularized \cite{zhang2008variable} by Sedumi  &  67.2(0.015) & 20.64 & 74.1 & 46.1 & 46.1 & 46.1 & 46.1
\\\hline
\end{tabular}
\end{center}}
\end{table*}

\subsubsection{Benchmark data}
In this subsection, we compare M-PGH to two popular methods for multicategory classification by binary-classifier. The first one is the ``one-vs-all'' (OVA) method \cite{bottou1994comparison} coded in LIBLINEAR library and another one the Decision Directed Acyclic Graph (DDAG) method\footnote{The code of DDAG is available from http://theoval.cmp.uea.ac.uk/svm/toolbox/} \cite{platt2000large}. We compared them on eight sets of benchmark data, all of which are available from the LIBSVM dataset. The problem statistics are shown in Table \ref{table:bench}. The original dataset of \textbf{connect4} has 67,557 data points. We randomly chose 500 for training and 1,000 for testing. All 2,000 data points in the training set of \textbf{dna} were used, and we randomly chose 500 for training and the rest for testing. \textbf{glass} has 214 data points, and we randomly chose 164 for training and another 50 for testing. 
For \textbf{letter}, we randomly picked 1300 out of the original 15,000 training samples with 50 for each class for training and 500 out of the original 5,000 testing points for testing. The \textbf{poker} dataset has 25,010 training and 1 million testing data points. For each class, we randomly selected 50 out of each class for training except the 6th through 9th classes which have less than 50 samples and hence were all used. In addition, we randomly chose 100k points from the testing set for testing. \textbf{protein} has 17,766 training data points and 6,621 testing points. We randomly chose 1,500 from the training dataset for training and all the points in the testing dataset for testing. 
 \textbf{usps} is a handwritten digit dataset consisting of 7291 training and 2007 testing digits from 0 to 9. We randomly picked 50 with 5 for each class out of the training set for training and 500 out of the testing set for testing. \textbf{wine} has 128 data points, and we randomly chose 50 for training and the rest for testing.

We fixed $\lambda_3=1$ in \eqref{eq:multihsvm} for M-PGH and tuned $\lambda_1,\lambda_2$. DDAG has two parameters $C,\gamma$, which were tuned from a large grid of values. The parameters for OVA were set to their default values in the code of LIBLINEAR. We did the tests for 10 times independently. The average prediction accuracies\footnote{Our reported accuracies are lower than those in \cite{hsu2002comparison} because we used fewer training samples.} by the three different methods are reported in Table \ref{table:bench-pred}. From the results, we see that M-PGH performs consistently better than OVA except for \textbf{letter} and also comparable with DDAG. When the training samples are sufficiently many, DDAG can give higher prediction accuracies such as for \textbf{glass} and \textbf{letter}. However, it can perform badly if the training samples are few compared to the feature numbers such as for \textbf{dna} and also the tests in the next subsection. In addition, note that the \textbf{poker} dataset is imbalanced, and our ``all-together'' model \eqref{eq:multihsvm} gives significantly higher accuracy than that of OVA. This suggests that "all-together" methods can perform better than "one-vs-all" in imbalanced datasets. We plan to investigate this further in a follow up paper.  

\begin{table}\caption{Statistics of eight benchmark datasets.}\label{table:bench}
{\footnotesize
\begin{center}
\resizebox{0.45\textwidth}{!}{\begin{tabular}{|c|c|c|c|c|}\hline
dataset & \#training & \#testing & \#feature & \#class\\\hline
\textbf{connect4} & 500 & 1000 & 126 & 3\\
\textbf{dna} & 500 & 1500 & 180 & 3 \\
\textbf{glass} & 164 & 50 & 9 & 6 \\
\textbf{letter} & 1300 & 500 & 16 & 26\\
\textbf{poker} & 352 & 100k & 10 & 10\\
\textbf{protein} & 1500 & 6621 & 357 & 3\\
\textbf{usps} & 50 & 500 & 256 & 10\\
\textbf{wine} & 50 & 128 & 13 & 3\\\hline
\end{tabular}
}
\end{center}
}
\end{table}

\begin{table}\caption{Prediction accuracies (\%) by different methods on benchmark data.}\label{table:bench-pred}
{\footnotesize
\begin{center}
\begin{tabular}{|c|c|c|c|}\hline
dataset & M-PGH & OVA & DDAG \\\hline
\textbf{connect4} & \textbf{53.28} & 51.20 & 53.07\\
\textbf{dna} & \textbf{92.82} & 89.04 & 33.71  \\
\textbf{glass} & 53.00 & 51.40 & \textbf{62.80} \\
\textbf{letter} & 43.24 & 65.90 & \textbf{80.96} \\
\textbf{poker} & \textbf{ 35.13} & 14.29 & 24.05\\
\textbf{protein} & \textbf{60.22} & 56.84 & 53.20 \\
\textbf{usps} & 74.44 & 73.46 & \textbf{76.28}\\
\textbf{wine} & \textbf{96.64} & 96.25 & 94.00 \\\hline
\end{tabular}
\end{center}
}
\end{table}

\subsubsection{Application to microarray classification}
This subsection applies M-PGH to microarray classifications. Two real data sets were used. One is the children cancer data set in \cite{khan2001classification}, which used cDNA gene expression profiles and classified the small round blue cell tumors (SRBCTs) of childhood into four classes: neuroblastoma (NB), rhabdomyosarcoma (RMS), Burkitt lymphomas (BL) and the Ewing family of tumors (EWS).  The other is the leukemia data set in \cite{golub1999molecular}, which used gene expression monitoring and classified the acute leuk-emias into three classes: B-cell acute lymphoblastic leuk-emia (B-ALL), T-cell acute lymphoblastic leukemia (T-ALL) and acute myeloid leukemia (AML). The original distributions of the two data sets are given in Table \ref{table:distribution}. Both the two data sets have been tested before on certain M-SVMs for gene selection; see \cite{wang20071,zhang2008variable} for example.

\begin{table*}\caption{Original distributions of SRBCT and leukemia data sets}\label{table:distribution}
{\footnotesize
\begin{center}
\begin{tabular}{|c|ccccc||cccc|}\hline
\multirow{2}{*}{Data set}&\multicolumn{5}{|c||}{SRBCT}&\multicolumn{4}{|c|}{leukemia}\\\cline{2-10}
& NB & RMS &BL & EWS &total & B-ALL & T-ALL & AML & total\\\hline
Training & 12 & 20 & 8 & 23 & 63 & 19 & 8 & 11 & 38\\
Testing & 6 & 5 & 3 & 6 & 20 & 19 & 1 & 14 & 34\\\hline
\end{tabular}
\end{center}}
\end{table*}

Each observation in the SRBCT dataset has dimension of $p=2308$, namely, there are 2308 gene profiles. We first standardized the original training data in the following way. Let $\mbfx^o=[\bfx_1^o,\cdots,\bfx_n^o]$ be the original data matrix. The standardized matrix $\mbfx$ was obtained by
$$\vspace{-0.1cm}x_{gj}=\frac{x_{gj}^o-\text{mean}(x_{g1}^o,\cdots,x_{gn}^o)}{\text{std}(x_{g1}^o,\cdots,x_{gn}^o)},\ \forall g, j.$$
Similar normalization was done to the original testing data. Then we selected the best parameters of each model by three-fold cross validation on the standardized training data. Finally, we put the standardized training and testing data sets together and randomly picked 63 observations for training and the remaining 20 ones for testing. The average prediction accuracy, running time (sec), number of nonzeros (NZ) and number of nonzero rows (NR) of 100 independent trials are reported in Table \ref{table:med-multireal}, from which we can see that all models give similar prediction accuracies. The $\ell_\infty$-regularized M-SVM gives denser solutions, and M-PGH is much faster than Sedumi.

The leukemia data set has $p=7,129$ gene profiles. We standardized the original training and testing data in the same way as that in last test. Then we rank all genes on the standardized training data by the method used in \cite{dudoit2002comparison}. Specifically, let $\mbfx=[\bfx_1,\cdots,\bfx_n]$ be the standardized data matrix. The relevance measure for gene $g$ is defined as follows:
$$\vspace{-0.2cm}R(g)=\frac{\sum_{i,j}I(y_i=j)(m_{g}^j-m_g)}{\sum_{i,j}I(y_i=j)(x_{gi}-m_{g}^j)},\
g = 1,\cdots,p,$$
where $m_g$ denotes the mean of $\{x_{g1},\cdots,x_{gn}\}$ and $m_{g}^j$ denotes the mean of $\{x_{gi}: y_i=j\}$. According to $R(g)$, we selected the 3,571 most significant genes. Finally, we put the processed training and tesing data together and randomly chose 38 samples for training and the remaining ones for testing. The process was independently repeated 100 times. We compared all the above five different methods for M-SVMs. Table \ref{table:med-multireal} summarizes the average results, which show that M-PGH is significantly faster than Sedumi and that model \eqref{eq:multihsvm} gives comparable prediction accuracies with relatively denser solutions. In addition, we note that DDAG performs very badly possibly because the training samples are too few.

\begin{table*}\caption{Comparison of computational results for the different methods on SRBCT and leukemia datasets}
\label{table:med-multireal}
{\footnotesize
\begin{center}
\begin{tabular}{|c|cccc|cccc|}\hline
 \multirow{2}{*}{Problems}& \multicolumn{4}{|c|}{SRBCT} & \multicolumn{4}{|c|}{leukemia}\\\cline{2-9}
&accu.(\%) & time & NZ & NR & accu.(\%) & time & NZ & NR\\\hline
\eqref{eq:multihsvm} by M-PGH & 98.6(0.051) & {\it 0.088} & 220.44 & 94.43 & \textbf{91.9}(0.049) & {\it 0.241} & 457.02 & 218.25  \\
$\ell_1$-regularized \cite{wang20071} by Sedumi &  \textbf{98.9}(0.022) & 13.82 & 213.67 & 96.71 &  85.2(0.063) & 11.21 & 82.41 & 40.00\\ 
$\ell_\infty$-regularized \cite{zhang2008variable} by Sedumi &  97.9(0.033) & 120.86 & 437.09 & 109.28 &  85.2(0.063) & 169.56 & 82.41 & 40.00 \\
OVA & 96.2 & --- & --- & --- & 81.5 & --- & --- & ---\\
DDAG & 72.0 & --- & --- & --- & 47.1 & --- & --- & ---\\\hline
\end{tabular}
\end{center}}
\end{table*}

\subsubsection{Statistical comparison}
As in section \ref{sec:stat-real}, we also performed statistical comparison of M-PGH, OVA, and DDAG on the benchmark datasets together with the two microarray datasets. Table \ref{table:wilcoxon-ben} shows their $z$-values and corresponding $p$-values of the Wilcoxon signed-rank test. From the table, we see that there is no significant difference between any pair of the three methods at $p$-$\text{value}<0.05$. However, M-PGH makes significant difference with OVA at $p$-$\text{value}<0.10$. Average ranks of M-PGH, OVA, and DDAG according to their prediction accuracies on the 10 datasets are shown in Table \ref{table:fried-ben} together with the Friedman statistic and $p$-value. Again, we see that there is no significant difference among the three methods at $p$-$\text{value}<0.05$ but there is at $p$-$\text{value}<0.10$. 
We further did a post-hoc test using the Holm's step-down procedure as in section \ref{sec:stat-real}. M-PGH was set as the control classifier. The $p$-$\text{values}$ are 0.0253 and 0.0736 respectively for OVA and DDAG. Hence, M-PGH made significant differences with OVA and DDAG at $p$-$\text{value}=0.10$.

\begin{table}\caption{$z$-value (above diagonal) and $p$-value (below diagnal) of Wilcoxon signed-ranks test of {M-PGH}, {OVA}, and {DDAG} on the eight benchmark and two microarray datasets.}\label{table:wilcoxon-ben}
\centering
{\footnotesize \begin{tabular}{|c|ccc|}\hline
Methods & M-PGH & OVA & DDAG\\\hline
M-PGH & --  & -1.7838  & -1.2741 \\
 OVA &   0.0745    &     --  & -0.5606 \\
 DDAG &   0.2026  &  0.5751   &      -- \\\hline
    \end{tabular}}
\end{table}

\begin{table}\caption{Average ranks of M-PGH, OVA, and DDAG according to their classification accuracies on the eight benchmark and two microarray datasets}\label{table:fried-ben}
\centering
\begin{tabular}{|c|ccc|}
\hline
Methods & M-PGH & OVA & DDAG\\\hline
Average rank & 1.4 & 2.4 & 2.2\\\hline
\multicolumn{4}{|c|}{$\chi_F^2$-$\text{value}=5.6$, $p$-$\text{value}=0.0608$}\\\hline
\end{tabular}
\end{table}

\section{Conclusions}\label{sec:conclusion}
SVMs have been popularly used to solve a wide variety of classification problems. The original SVM model uses the non-differentiable hinge loss function, which together with some regularizers like $\ell_1$-term makes it difficult to develop simple yet efficient algorithms. We considered the huberized hinge loss function that is a differentiable approximation of the original hinge loss. This allowed us to apply PG method to both binary-class and multi-class SVMs in an efficient and accurate way. In addition, we presented a two-stage algorithm that is able to solve very large-scale binary classification problems. Assuming strong convexity and under fairly general assumptions, we proved the linear convergence of PG method when applied in solving composite problems in the form of \eqref{eq:comp}, special cases of which are the binary and multi-class SVMs. We performed a wide range of numerical experiments on both synthetic and real datasets, demonstrating the superiority of the proposed algorithms over some state-of-the-art algorithms for both binary and multi-class SVMs. In particular for large-scale problems, our algorithms are significantly faster than compared methods in all cases with comparable accuracy. Finally, our algorithms are more robust to the smoothing parameter $\delta$ in terms of CPU time.


\section*{Acknowledgements}
The authors would like to thank four anonymous referees and the associate editor for their valuable comments and suggestions that improved the quality of our paper. The first author would also like to thank Professor Wotao Yin for his valuable discussion.

\appendix
\section{Proof of Proposition \ref{prop:f}}
First, we derive the convexity of each $f_i$ from the composition of the convex function
$\phi_H$ and the linear transformation $y_i(b+\bfx_i^\top \bfw)$ (e.g., see \cite{BoydVandenberghe04}). Thus, $f$ is also convex since it is the sum of $n$ convex functions. Secondly, it is easy to verify that 
$$\phi'_H(t)=\left\{
\begin{array}{ll}
0,&\text{ for }t>1;\\
\frac{t-1}{\delta},&\text{ for }1-\delta<t\le 1;\\
-1,&\text{ for }t\le 1-\delta;
\end{array}
\right.$$
and it is Lipschitz continuous, namely, $|\phi'_H(t)-\phi'_H(s)|\le\frac{1}{\delta}|t-s|,$ for any $t,s$. 

Using the notations in (\ref{notation_uv}) and by the chain rule, we have $\nabla f_i(\bfu)=\phi'_H(\bfv_i^\top \bfu)\bfv_i$, for $i=1,\cdots,n$ and
\begin{equation}\label{grad}
\nabla f(\bfu)=\frac{1}{n}\sum_{i=1}^n\phi'_H(\bfv_i^\top \bfu)\bfv_i.
\end{equation}
 Hence, for any $\bfu,\hat{\bfu}$, we have
$$
\begin{array}{rl}
&\|\nabla f(\bfu)-\nabla f(\hat{\bfu})\|\\[0.2cm]
\le&~\hspace{-0.2cm}\frac{1}{n}\sum_{i=1}^n\left\|\nabla f_i(\bfu)-\nabla f_i(\hat{\bfu})\right\|\\[0.2cm]
=&\frac{1}{n}\sum_{i=1}^n\left\|\left(\phi'(\bfv_i^\top \bfu)-\phi'(\bfv_i^\top \hat{\bfu})\right)\bfv_i\right\|\\[0.2cm]
\le&\frac{1}{n\delta}\sum_{i=1}^n|\bfv_i^\top (\bfu-\hat{\bfu})|\|\bfv_i\|\\[0.2cm]
\le & \frac{1}{n\delta}\sum_{i=1}^n \|\bfv_i\|^2\|\bfu-\hat{\bfu}\|,
\end{array}$$
which completes the proof.

\section{Proof of Theorem \ref{thm:finitesquare}}
We begin our analysis with the following lemma, which can be shown through essentially the same proof of Lemma 2.3 in \cite{BeckTeboulle2009}.
\begin{lemma}\label{lem:key}
Let $\{\bfu^k\}$ be the sequence generated by \eqref{eq:pg} with $L_k\le L$ such that for all $k\ge 1$,
\begin{equation}\label{eq:condL}
\begin{array}{rl}\xi_1({\bfu}^{k})\le &\xi_1(\hat{\bfu}^{k-1})
+\left\langle\nabla\xi_1(\hat{\bfu}^{k-1}), {\bfu}^{k}-\hat{\bfu}^{k-1}\right\rangle\\
&\hspace{1cm}+\frac{L_k}{2}\left\|{\bfu}^{k}-\hat{\bfu}^{k-1}\right\|^2.\end{array}
\end{equation}
Then for all $k\ge1$, it holds for any $\bfu\in\mcu$ that
\begin{equation}\label{eq:key}
\begin{array}{ll}\xi(\bfu)-\xi(\bfu^k)\ge& \frac{L_k}{2}\|\bfu^k-\hat{\bfu}^{k-1}\|^2+\frac{\mu}{2}\|\bfu-\bfu^k\|^2
\\[0.1cm]
&\hspace{0.5cm}+L_k\langle \hat{\bfu}^{k-1}-\bfu, \bfu^k-\hat{\bfu}^{k-1}\rangle.
\end{array}
\end{equation}
\end{lemma}

We also need the following lemma, which is the KL property for strongly convex functions.

\begin{lemma}\label{lem:KL}
Suppose $\xi$ is strongly convex with constant $\mu>0$.
Then for any $\bfu,\bfv\in\mathrm{dom}(\partial\xi)$ and $\bfg_\bfu\in\partial\xi(\bfu)$, we have
\begin{equation}\label{eq:KL}
\xi(\bfu)-\xi(\bfv)\le\frac{1}{\mu}\|\bfg_\bfu\|^2.
\end{equation}
\end{lemma}

\begin{proof}
For any $\bfg_\bfu\in\partial\xi(\bfu)$, we have from the strong convexity of $\xi$ and the Cauchy-Schwarz inequality that
$$\xi(\bfu)-\xi(\bfv)\le\langle\bfg_\bfu,\bfu-\bfv\rangle-\frac{\mu}{2}\|\bfu-\bfv\|^2\le\frac{1}{\mu}\|\bfg_\bfu\|^2,$$
which completes the proof.
\end{proof}

\section*{Global convergence}
We first show $\bfu^k\to\bfu^*$.
Letting $\bfu=\bfu^{k-1}$ in \eqref{eq:key} gives
\begin{align}
&\xi(\bfu^{k-1})-\xi(\bfu^k)\cr
\ge &\frac{L_k}{2}\|\bfu^k-\hat{\bfu}^{k-1}\|^2+\frac{\mu}{2}\|\bfu^{k-1}-\bfu^k\|^2\cr
&\hspace{0.5cm}+L_k\langle \hat{\bfu}^{k-1}-\bfu^{k-1}, \bfu^k-\hat{\bfu}^{k-1}\rangle\nonumber\\
=&\frac{L_k}{2}\|\bfu^k-{\bfu}^{k-1}\|^2+\frac{\mu}{2}\|\bfu^k-{\bfu}^{k-1}\|^2\cr
&\hspace{0.5cm}-\frac{L_k}{2}\omega_{k-1}^2\|\bfu^{k-1}-{\bfu}^{k-2}\|^2\nonumber\\
\ge&\frac{L_k}{2}\|\bfu^k-{\bfu}^{k-1}\|^2+\frac{\mu}{2}\|\bfu^k-{\bfu}^{k-1}\|^2\label{eq:diff}\\
&\hspace{0.5cm}-\frac{L_{k-1}}{2}\|\bfu^{k-1}-{\bfu}^{k-2}\|^2.\nonumber
\end{align}
Summing up the inequality \eqref{eq:diff} over $k$, we have
$$\begin{array}{rl}&\xi(\bfu^0)-\xi(\bfu^K)\\[0.1cm]
\ge&\frac{\mu}{2}\sum_{k=1}^K\|\bfu^{k-1}-\bfu^k\|^2
+\frac{L_K}{2}\|\bfu^K-\bfu^{K-1}\|^2,
\end{array}$$
which implies $\bfu^k-\bfu^{k-1}\to\mathbf{0}$ since $\xi(\bfu^k)$ is lower bounded.

Note that $\{\bfu^k\}$ is bounded since $\xi$ is coercive and $\{\xi(\bfu^k)\}$ is upper bounded by $\xi(\bfu^0)$. Hence, $\{\bfu^k\}$ has a limit point $\bar{\bfu}$, so there is a subsequence $\{\bfu^{k_j}\}$ converging to $\bar{\bfu}$. Note $\bfu^{k_j+1}$ also converges to $\bar{\bfu}$ since $\bfu^k-\bfu^{k-1}\to\mathbf{0}$. Without loss of generality, we assume $L_{k_j+1}\to \bar{L}$. Otherwise, we can take a convergent subsequence of $\{L_{k_j+1}\}$. By \eqref{eq:pg}, we have
$$\begin{array}{rl}\bfu^{k_j+1}=&\arg\underset{\bfu}{\min}
\xi_1(\hat{\bfu}^{k_j})+
\langle\nabla \xi_1(\hat{\bfu}^{k_j}),\bfu-\hat{\bfu}^{k_j}\rangle\\[0.1cm]
&\hspace{1cm}+\frac{L_{k_j+1}}{2}\|\bfu-\hat{\bfu}^{k_j}\|^2+\xi_2(\bfu).
\end{array}
$$
Letting $j\to\infty$ in the above formula and observing $\hat{\bfu}^{k_j}\to\bar{\bfu}$ yield
$$\begin{array}{rl}\hspace{-0.2cm}\bar{\bfu}=\arg\underset{\bfu}{\min}&\hspace{-0.2cm} \xi_1(\bar{\bfu})+
\langle\nabla \xi_1(\bar{\bfu}),\bfu-\bar{\bfu}\rangle
+\frac{\bar{L}}{2}\|\bfu-\bar{\bfu}\|^2+\xi_2(\bfu),
\end{array}
$$
which indicates $-\nabla \xi_1(\bar{\bfu})\in\partial\xi_2(\bar{\bfu})$. Thus $\bar{\bfu}$ is the minimizer of $\xi$.

Since $\{\xi(\bfu^k)\}$ is nonincreasing and lower bounded, it converges to $\xi(\bar{\bfu})=\xi(\bfu^*)$. Noting 
$\frac{\mu}{2}\|\bfu^k-\bfu^*\|^2\le\xi(\bfu^k)-\xi(\bfu^*),$
we have $\bfu^k\to\bfu^*$.

\section*{Convergence rate}
Next we go to show \eqref{eq:ratefista}.
Without loss of generality, we assume $\xi(\bfu^*)=0$. Otherwise, we can consider $\xi-\xi(\bfu^*)$ instead of $\xi$. 
In addition, we assume $\xi(\bfu^k)>\xi(\bfu^*)$ for all $k\ge0$ because if $\xi(\bfu^{k_0})=\xi(\bfu^*)$ for some $k_0$, then $\bfu^k=\bfu^{k_0}=\bfu^*$ for all $k\ge k_0$.

For ease of description, we denote $\xi_k=\xi(\bfu^k)$. 
Letting $\bfv=\bfu^*$ in \eqref{eq:KL}, we have
\begin{equation}\label{temp1}
\sqrt{\xi_k}\le\frac{1}{\sqrt{\mu}}\|\bfg_{\bfu^k}\|, \text{ for all }\bfg_{\bfu^k}\in\partial \xi(\bfu^k).
\end{equation}
Noting $-\nabla \xi_1(\hat{\bfu}^{k-1})-L_k(\bfu^k-\hat{\bfu}^{k-1})+\nabla \xi_1(\bfu^k)\in\partial \xi(\bfu^k)$ and $L_k\le L$,
we have for all $k\ge K$,
\begin{equation}\label{temp2}
\sqrt{\xi_k}\le\frac{2L}{\sqrt{\mu}}\left(\|\bfu^k-{\bfu}^{k-1}\|+\|\bfu^{k-1}-{\bfu}^{k-2}\|\right)
\end{equation}
Noting $\sqrt{\xi_k}-\sqrt{\xi_{k+1}}\ge \frac{\xi_k-\xi_{k+1}}{2\sqrt{\xi_k}}$ and using \eqref{eq:diff} yield 
$$\begin{array}{ll}&(L_{k+1}+\mu)\|\bfu^{k}-\bfu^{k+1}\|^2\\[0.1cm]
\le &L_k\|\bfu^{k-1}-\bfu^{k}\|^2+\frac{8L}{\sqrt{\mu}}\left(\sqrt{\xi_k}-\sqrt{\xi_{k+1}}\right)\|\bfu^k-{\bfu}^{k-1}\|\\[0.1cm]
&+\frac{8L}{\sqrt{\mu}}\left(\sqrt{\xi_k}-\sqrt{\xi_{k+1}}\right)\|\bfu^{k-1}-{\bfu}^{k-2}\|,\end{array}$$
after rearrangements. Take $0<\delta<\frac{1}{2}(\sqrt{L+\mu}-\sqrt{L})$. Using the inequalities $a^2+b^2\le (a+b)^2$ and $\sqrt{ab}\le ta+\frac{1}{4t}b$ for $a,b,t>0$, we have from the above inequality that
$$\begin{array}{ll}&\sqrt{L_{k+1}+\mu}\|\bfu^{k}-\bfu^{k+1}\|\\[0.1cm]
\le&\sqrt{L_k}\|\bfu^{k-1}-\bfu^{k}\|+\frac{2L}{\delta\sqrt{\mu}}\left(\sqrt{\xi_k}-\sqrt{\xi_{k+1}}\right)\\[0.1cm]
&\hspace{0.2cm}+\delta\left(\|\bfu^k-{\bfu}^{k-1}\|+\|\bfu^{k-1}-{\bfu}^{k-2}\|\right).\end{array}$$
Summing up the above inequality over $k$ and noting $\|\bfu^k-\bfu^{k+1}\|\to 0, \xi_k\to 0$ yield
$$\begin{array}{ll}&\sum_{k=m}^\infty\left(\sqrt{L_{k+1}+\mu}-\sqrt{L_{k+1}}-2\delta\right)\|\bfu^k-\bfu^{k+1}\|\\[0.1cm]
\le& \left(\sqrt{L}+2\delta\right)\|\bfu^{m-1}-\bfu^m\|+\delta\|\bfu^{m-2}-\bfu^{m-1}\|+\frac{2L}{\delta\sqrt{\mu}}\sqrt{\xi_m},\end{array}
$$
which together with $\sqrt{L+\mu}-\sqrt{L}\le \sqrt{L_k+\mu}-\sqrt{L_k}$ for all $k\ge0$ implies
\begin{equation}\label{temp3}
\begin{array}{ll}
&\sum_{k=m}^\infty\|\bfu^k-\bfu^{k+1}\|\\[0.1cm]
\le & C_1 \sqrt{\xi_m}+C_2\left(\|\bfu^{m-1}-\bfu^m\|+\|\bfu^{m-2}-\bfu^{m-1}\|\right),
\end{array}
\end{equation}
where
$$\begin{array}{l}C_1=\frac{2L}{\delta\sqrt{\mu}\left(\sqrt{L+\mu}-\sqrt{L}-2\delta\right)},\ C_2=\frac{\sqrt{L}+2\delta}{\sqrt{L+\mu}-\sqrt{L}-2\delta}.\end{array}$$
Denote $S_m=\sum_{k=m}^\infty\|\bfu^k-\bfu^{k+1}\|$ and write \eqref{temp3}
as
$$S_m\le C_1 \sqrt{\xi_m}+C_2(S_{m-2}-S_m),$$
which together with \eqref{temp2} gives
$$\begin{array}{rl}S_m\le &\left(C_1\frac{2L}{\sqrt{\mu}}+C_2\right)(S_{m-2}-S_m)\\[0.1cm]
=&C_3(S_{m-2}-S_m),\text{ for all }m\ge 2.
\end{array}$$
Let $\tau=\sqrt{\frac{C_3}{1+C_3}}$ and $C=S_0$. Then we have
$S_m\le C\tau^m, \forall m\ge0.$
Note $\|\bfu^m-\bfu^*\|\le S_m$, and thus we complete the proof.

\bibliographystyle{plain}

\end{document}